\documentclass[accepted]{uai2026} 
                        
\usepackage[american]{babel}

\usepackage{natbib} 
\usepackage{mathtools} 
\usepackage{booktabs} 
\usepackage{tikz} 

\usepackage{subcaption}

\usepackage{amsfonts}
\usepackage{amsthm}
\usepackage{amsmath}  
\usepackage{bm}       
\usepackage[table, dvipsnames]{xcolor}
\usepackage{algorithm}
\usepackage{algpseudocode}

\usepackage{colortbl}
\usepackage{xcolor}

\usepackage[capitalize,noabbrev]{cleveref}
\usepackage{thmtools}
\usepackage{thm-restate}
\usepackage{xurl}

\definecolor{topcell}{gray}{0.9} 

\newtheorem*{definition}{Definition}

\newtheorem{remark}{Remark}
\newtheorem*{remark*}{Remark}
\newtheorem{lemma}{Lemma}

\usepackage{booktabs}   
\usepackage{multirow}   
\usepackage{graphicx}   
\usepackage{array}

\usepackage{graphicx}
\usepackage{breqn}
\usepackage{tabularx}
\usepackage{siunitx} 
\usepackage{import}
\usepackage{caption}

\newcommand{\E}{\mathbb{E}}
\newcommand{\reals}{\mathbb{R}}
\newcommand{\opA}{\mathcal{A}}

\newcommand{\rv}{\bm{r}}
\newcommand{\betav}{\bm{\beta}}

\newcommand{\norm}[2]{\| #1 \|_{#2}}
\newcommand{\abs}[1]{\left|#1\right|}
\newcommand{\midand}{\quad \text{and} \quad}

\title{Finding the Signal in the Spam: Jointly Learning \\ Rewards and Worker Reliability from Pairwise Comparisons}

\author[1,*]{\href{mailto:kaustubhshejole@cse.iitb.ac.in?Subject=UAI 2026}{Kaustubh Shivshankar Shejole}{}}
\author[1,*]
{\href{mailto:tanishagarwal2910@gmail.com?Subject=UAI 2026}{Tanish Agarwal}{}}
\author[1]{Arpit Agarwal}
\author[1]{Avishek Ghosh}

\affil[*]{Equal contribution}
\affil[1]{%
    IIT Bombay\\
    Mumbai, India -- 400076
}

\begin{document}
\maketitle

\begin{abstract}
The problem of learning from pairwise comparisons has been widely studied across many domains such as recommendation systems, social choice, and more recently, fine-tuning large language models. In this problem, the goal is to learn item rewards based on pairwise comparisons between them. 
In many scenarios, these comparisons are elicited from crowdworkers using platforms such as Amazon Mechanical Turk, Scale AI, etc.
However, crowdworkers are often unreliable
due to limited
domain knowledge or revenue-maximizing (spamming) behavior.
In this work, 
our goal is to understand 
whether worker reliability (competency) can be learned jointly with item rewards.
To this end, we adopt the Boltzmann-rational model for pairwise comparisons, which extends the Bradley–Terry–Luce model by incorporating worker competencies.
We derive an EM-based algorithm for learning under this model by introducing Polya-Gamma latent variables
to transform the logistic likelihood into a conditionally Gaussian form, enabling tractable optimization and leading to a simplified $Q$ function in the E-step of the algorithm. This technique allows us to reduce our formulation to a matrix sensing problem, using which we establish theoretical convergence guarantees for our algorithm.
We conduct extensive experiments on real-world and synthetic datasets. 
These experiments demonstrate the advantages of using our algorithm over several baselines and confirm its strong robustness to both spammers and adversarial workers, highlighting its practical effectiveness in realistic crowdsourcing and reward learning settings.\footnote{The code and data is publicly available at \url{https://github.com/KaustubhShejole/BoRa_EM}.}

\end{abstract}
\section{Introduction}\label{sec:intro}
Reward learning from pairwise comparisons is a widely studied problem in machine learning, with applications in recommendation systems, information retrieval, social choice, and, more recently, fine-tuning of large language models.
In this problem, given pairwise comparisons between $K$ items,
the goal is to learn a reward for each item that is ``consistent'' with the observed comparisons.
For example, in recommendation systems, one observes the choices of users over recommended items and the goal is to learn a score for each item \citep{kalloori2018eliciting};
in social choice, one observes preferences of individuals over various policy alternatives and the goal is to aggregate
the preferences into a single ranked list \citep{list2013social}; and in fine-tuning large language models, one observes pairwise comparisons between different model responses and the goal is to learn a reward for each response \citep{bakker2022fine, 10.5555/3495724.3495977}.

Given the large-scale requirement of preference data in these applications,  comparisons between items are often elicited using human crowdworkers  \citep{velayati2025crowdsourcing, chen2013pairwise}.
Due to this, many crowdsourcing marketplaces such as Amazon Mechanical Turk, Scale AI, Figure Eight, and Yandex Toloka have gained prominence in recent years.
However, it has been observed that crowdworkers often exhibit large variability in the quality of labels \citep{gadiraju2015understanding, ding2022effectiveness}. There may be several factors leading to this variability such as insufficient knowledge of the domain, lack of time/effort devoted to the tasks, and revenue maximizing/spamming behavior. The default approach adopted by many platforms is to identify a gold-standard dataset and evaluate worker competency on it \citep{chen2013pairwise}. However, in many scenarios such datasets are not available, and requiring workers to answer gold-standard questions can be inefficient and costly.


In this work, we study the problem of jointly estimating item rewards and worker competencies from pairwise comparison data, without requiring a gold-standard dataset. We adopt the Boltzmann-rational model, in which each item $j \in [K]$ has a latent reward $r_j$ and each worker $s \in [M]$ has a competency parameter $\beta_s \in [-1,1]$. The probability that worker $s$ prefers item $w$ over item $l$ is given by
\[
P[w \succ l \mid s] = \sigma\!\bigl(\beta_s(r_w - r_l)\bigr),
\]
where $\sigma$ denotes the logistic function. This model captures a spectrum of worker behaviors: $\beta_s = 1$ corresponds to an ideal annotator, $\beta_s = 0$ to a random spammer, and $\beta_s = -1$ to an adversarial worker. Given a dataset $\mathcal{D} = \{(w_i, l_i, s_i)\}_{i=1}^N$ of $N$ pairwise comparisons, the goal is to jointly estimate the reward vector $\mathbf{r} \in \mathbb{R}^K$ and the competency vector $\boldsymbol{\beta} \in [-1,1]^M$.

We derive an expectation-maximization (EM) algorithm for inference under this model. To derive the EM algorithm, we use a P\'{o}lya-Gamma augmentation \citep{polson2013bayesian}, which introduces auxiliary variables  to transform the logistic likelihood into a conditionally Gaussian form. This leads to a quadratic objective for the M-step which enables efficient alternating maximization over $\boldsymbol{\beta}$ and $\mathbf{r}$ in closed form (Section~\ref{sec:polya_gamma_augmentation}).
While this alternating minimization can be solved efficiently, it is not guaranteed to converge to a global optimum. 

To address this, we show that the M-step optimization can be cast as a rank-1 matrix sensing problem. Specifically, defining $X = \boldsymbol{\beta}\mathbf{r}^\top \in \mathbb{R}^{M \times K}$ and a linear operator $\mathcal{A}: \mathbb{R}^{M \times K} \to \mathbb{R}^N$, the M-step reduces to minimizing $\|\mathcal{A}(\boldsymbol{\beta}\mathbf{r}^\top) - \mathbf{u}\|_2^2$. We establish that $\mathcal{A}$ satisfies the Restricted Isometry Property (RIP) given a sufficient number of samples. Leveraging results from non-convex low-rank matrix recovery \citep{park2016nonsquarematrixsensingspurious}, we prove that under the RIP condition, at each iteration of the EM algorithm, every local minimum of the M-step objective is near-global minimum. Using standard results for EM, we are able to show that the likelihood increases monotonically and converges to a stationary point of the likelihood function.

We conduct extensive experiments on both real-world and synthetic datasets. The experimental results demonstrate that our algorithm outperforms baseline methods, and more importantly, it remains robust in scenarios with large numbers of spammers of various types and achieves significant performance gains over other existing approaches in this case. Thus, we provide a unified framework capable of handling diverse spammer behaviors---from random clickers to malicious adversaries---offering a practical and theoretically grounded alternative for real-world crowdsourcing platforms where worker quality is heterogeneous and unknown \textit{a priori}.

\section{Related Work}
\label{sec:related_work}
Learning rewards from pairwise comparisons has been widely used across many domains such as economics, operations research and machine learning. The classical Bradley-Terry-Luce (BTL) model~\citep{bradley1952rank} assumes that each item 
is associated with an underlying reward, and pairwise comparisons between items are drawn by taking into account the reward difference between them.
Specifically, the probability that
item $w$ is preferred over item $l$ is given as
\(    P(w \succ l) = \sigma(r_{w} - r_{l}),
\)
where $\sigma(\cdot)$ denotes the sigmoid function, and $\mathbf{r} = \{r_j\}_{j=1}^K$ represents the latent reward or score associated with each of the $K$ items.
The BTL model, however, does not capture the heterogeneity in the rewards observed by different crowdworkers.

\citet{negahban2012iterative} proposed RankCentrality (\texttt{RC}), which interprets pairwise comparisons as a directed graph, where nodes represent items and directed edges represent comparison outcomes. This formulation enables the use of random-walk–based methods for ranking items.


The closest to our work is Crowd-BT model \citep{chen2013pairwise} which extends the \texttt{BTL} model by incorporating worker reliability. For a worker $s$, the probability of preferring item $i$ over $j$ is modeled as $\Pr(i \succ j \mid s) = \beta_s \, \sigma(r_i - r_j) + (1-\beta_s) \, \sigma(r_j - r_i)$, where $r_i, r_j$ are item rewards and $\beta_s$ captures the worker's competency. 
This model is effectively a mixture of two \texttt{BTL} models -- one with reward $\mathbf{r}$ and the other with negative reward i.e., $-\mathbf{r}$.
The mixture weight represents the competency of the crowdworker, with $\beta_s = 1/2$ representing a worker that labels uniformly at random.
Crowd-BT is an alternating minimization approach to jointly learn the rewards and competencies.
However, \citet{chen2013pairwise} do not provide any theoretical guarantees, and they require a gold standard dataset to initialize worker competencies, which is not always available.
On the other hand, our work adopts the Boltzmann rational choice model, which provides a more natural and theoretically grounded formulation for preference-based reward learning and alignment, and allows us to prove theoretical convergence results.

Another related work is \citet{bugakova2019aggregation}, which introduced the Factor-BT model to explicitly capture systematic biases present in crowdsourced pairwise comparisons.
Unlike prior models, which treat task-related biases as noise or ignore them altogether, \texttt{FactorBT} explicitly models these biases by incorporating known but irrelevant task features (e.g., item position on screen, background design) that may influence worker decisions. 
For a comparison between items $d_i$ and $d_j$ by worker $s$, described by a feature vector $\mathbf{x}_{sij} \in \mathbb{R}^M$, where $M$ is the feature dimension. The probability that $d_i$ is preferred over $d_j$ is given  as follows \(
\Pr(d_i \succ_s d_j) = \sigma(\gamma_s) \cdot \sigma(r_i - r_j) + (1 - \sigma(\gamma_s)) \cdot \sigma(\langle \boldsymbol{b}_s, \mathbf{x}_{sij} \rangle),
\)
where $r_i, r_j$ are the latent scores of items $d_i$ and $d_j$, $\gamma_s$ controls the susceptibility of worker $s$ to bias and $\sigma(\gamma_s)$ can be considered as the the skill of the worker same as that of $\beta_s$ in our formulation, $\boldsymbol{b}_s$ encodes the worker-specific sensitivity to task features, $\sigma(\cdot)$ is the sigmoid function. The model estimates parameters $\{r_i\}$, $\{\gamma_s\}$, and $\{\boldsymbol{b}_s\}$ by maximizing the regularized log-likelihood of the observed pairwise comparisons.
Unlike \texttt{FactorBT} \citep{bugakova2019aggregation}, we do not explicitly model the influence of task-irrelevant factors such as item position, and we instead focus on learning worker competency jointly with rewards.


\citet{jin2020rank} proposed two models that explicitly account for annotator 
competence: the Heterogeneous \emph{Bradley\--Terry\--Luce} (HBTL) model and 
the Heterogeneous \emph{Thurstone Case~V} (HTCV) model. The HBTL model is based on the Boltzmann-Rational formulation given in Equation~\eqref{eq:boltz}. In contrast, the HTCV model assumes that the probability of item $w$ being preferred over item $l$ by annotator $s$ is given by
\begin{equation}
    P(w \succ l; s) = \Phi\left(\frac{\beta_s (r_w - r_l)}{\sqrt{2}}\right),
\end{equation}

where $\Phi(\cdot)$ denotes the cumulative distribution function (CDF) of the standard normal distribution, and $\beta_s$ captures the competence of annotator $s$. They perform alternating gradient descent for both these methods and provide theoretical convergence guarantees.

Bias-Aware Ranking from Pairwise comparisons (BARP) \citep{ferrara2024bias} extends the classic BTL model to account for evaluator biases in pairwise comparisons. Each evaluator is assigned a bias parameter that distorts the latent quality scores of items based on the group membership of items. By explicitly modeling the log-likelihood over items' latent scores and evaluators' biases, \texttt{BARP} uses an alternating optimization approach to jointly estimate both. In \texttt{BARP}, the probability that evaluator $s$ prefers item $i$ over $j$ is modeled as \(
P(i \succ_s j) \;=\; \sigma\!\left( (r_i - r_j) + \theta_s \, (\delta_{ig} - \delta_{jg}) \right),
\)
where $r_i, r_j$ are the latent item scores and $\theta_s \delta_{ig}$ represents the bias of evaluator $s$ for the group $g$ of item $i$. Unlike the standard BTL model, this formulation accounts for biased perception of item scores rather than true latent scores.


\paragraph{Organization} Section~\ref{sec:problem_formulation} introduces the problem of jointly learning rewards and worker competence from pairwise comparisons. Section~\ref{sec:polya_gamma_augmentation} presents the derivation of our proposed method (\textbf{BoRaEM}) for the \textbf{Bo}ltzmann \textbf{Ra}tional Model using Expectation--Maximization (\textbf{EM}), while Section~\ref{sec:theoretical_guarantees} establishes theoretical guarantees for its convergence. Section~\ref{sec:experimental_analysis} reports experimental results and analysis. Finally, Section~\ref{sec:conclusion} concludes with a discussion of future research directions.

\section{PROBLEM FORMULATION}
\label{sec:problem_formulation}
We consider a setting where there is a set of $K$ items, and one observes outcomes of pairwise comparisons
between these items, elicited from a set of $M$ crowdworkers.
Specifically, we are given a dataset $\mathcal{D} = \{(w_i, l_i, s_i)\}_{i=1}^N$ where there are $N$ pairwise comparisons; the $i$-th comparison is performed by worker $s_i \in [M]$ where $w_i \in [K]$ is the winning item and $l_i \in [K]$ is the losing item such that $l_i$ is not equal to $w_i$.

We assume that the outcomes of these comparisons are generated according to the Boltzmann-rational model
which is an extension of the \texttt{BTL} model and has received a lot of attention recently \citep{daniels2022expertise, jeon2020reward, ziebart2010maximum}.
Similar to the \texttt{BTL} model, each item $j \in [K]$ is associated with a reward $r_j$, 
and the probability of winning a comparison is governed by pairwise reward differences.
This model also considers the heterogeneity of human decision-making. 
Specifically, humans are considered rational decision-makers, where the quality of their decisions is modulated by a ``rationality'' parameter $\beta \in [-1,1]$ \citep{daniels2022expertise, jeon2020reward, ziebart2010maximum}. 
Formally, the probability that item $w$ beats item $l$ given worker $s$ is given by
\begin{dmath}
P[w \succ l; s] 
= \frac{\exp(\beta_s r_w)}{\exp(\beta_s r_w) + \exp(\beta_s r_l)}
= \frac{1}{1 + \exp(-\beta_s (r_w - r_l))}
= \sigma(\beta_s (r_w - r_l)).
\label{eq:boltz}
\end{dmath}
The case $\beta_s = 1 $ corresponds to a perfect annotator who makes decisions according to the underlying \texttt{BTL} model, $\beta_s  = -1$ for an ``adversarial'' worker, and  $\beta_s = 0$ for a ``spammer'' who labels uniformly at random. 


Given a dataset of pairwise comparisons $\mathcal{D} = \{(w_i, l_i, s_i)\}_{i=1}^N$ that is drawn according to the 
Boltzmann-rational model with reward vector $\mathbf{r} = \{r_j\}_{j=1}^K$ and worker competency vector $\boldsymbol{\beta} = \{\beta_s\}_{s=1}^M$, 
the goal is to jointly estimate $\mathbf{r}$ and $\boldsymbol{\beta}$.
Since the reward model is invariant to both additive shifts and scaling, as the likelihood depends only on the product $\beta_s (r_w - r_l)$, the transformations $r \rightarrow r + c$ and $r \rightarrow c r$, $\beta_s \rightarrow \beta_s / c$ leave the likelihood unchanged, resulting in non-identifiability.
To address this, we enforce identifiability by centering the rewards to remove translation invariance and normalizing them to unit root-mean-square after each EM iteration, with a compensatory rescaling of worker competencies. This fixes the latent scale and selects a unique representative solution. Additionally, mild $\ell_2$ regularization on rewards and a Gaussian prior on competencies improve numerical stability and prevent degenerate scaling. More details about implementation are provided in Appendix \ref{app:pgem_implementation}.

\section{BoRaEM Algorithm}
\label{sec:polya_gamma_augmentation}

In order to solve this problem, one can formulate a log-likelihood objective based on the underlying model.
Specifically, following Equation~\eqref{eq:boltz} for the probability that a worker $s_i$ favors $w_i$ over $l_i$:
\begin{equation}
    p(w_i \succ l_i \mid s_i) =  \sigma(\beta_{s_i}(r_{w_i} - r_{l_i}))
\label{eq:extended_btl}
\end{equation}
Let $\bm{\theta} = \mathbf{r}, \boldsymbol{\beta}$.
The associated complete-data likelihood is
\begin{equation}
\mathcal{L_{\text{complete}}}(\bm{\theta}; \mathcal{D}) = \prod_{({w_i}, {l_i}, s_i) \in \mathcal{D}} \sigma\left(\beta_{s_i}(r_{w_i} - r_{l_i})\right)
\label{eq:complete_data_log_likelihood}
\end{equation}
where $\mathcal{D}$ refers to the comparisons data.

However, direct maximum likelihood estimation (MLE) of the parameter vectors $\mathbf{r}$ and $\boldsymbol{\beta}$ is challenging as the objective is non-concave.
This complicates both optimization and theoretical analysis.


To circumvent this issue, we adopt the \textbf{Pólya-Gamma (PG) augmentation 
framework} \citep{polson2013bayesian}\footnote{We encourage readers to also refer to a useful blog on Pólya-Gamma (PG) augmentation at \url{https://gregorygundersen.com/blog/2019/09/20/polya-gamma/}.}. For a latent variable $\omega \sim \text{PG}(b, 0)$ 
where $b > 0$, and any $\psi \in \mathbb{R}$, the logistic likelihood admits the 
integral representation:
\begin{equation}
\label{eq:pg_identity}
    \frac{(e^{\psi})^{a}}{(1+e^{\psi})^{b}}= \frac{e^{\kappa \psi}}{2^b} \int_0^{\infty} \exp\left(-\frac{\omega \psi^2}{2}\right) p(\omega \mid b, 0) \, d\omega,
\end{equation}
with $\kappa = a - b/2$, where $p(\cdot)$ is the density of $\omega$.



We apply this identity to our model by setting $a=1$, $b=1$, and $\psi = \eta_i$ where $\eta_i := \beta_{s_i}(r_{w_i} - r_{l_i})$. 

Thus, the probability of the $i^{th}$ pairwise comparison can be written as:
\begin{align}
p(w_i, l_i, s_i | \bm{\theta}) &= \sigma(\eta_i) \notag\\
&= \frac{1}{2} e^{\eta_i / 2} \int_0^{\infty}\!\exp\left(-\frac{\omega_i \eta_i^2}{2}\right) p(\omega_i)  d\omega_i,
\label{eq:our_model_pg}
\end{align}
where $\omega_i \sim \text{PG}(1, 0)$ is a latent Pólya-Gamma variable introduced for each $i^{\text{th}}$ comparison.

So,
the joint density $p(w_i, l_i, s_i, \omega_i| \bm{\theta})$ is given by

\begin{equation}
p(w_i, l_i, s_i, \omega_i| \bm{\theta}) \propto \exp\left(-\frac{\omega_i \eta_i^2}{2} + \frac{\eta_i}{2}\right) p(\omega_i)
\label{eq:joint_density}
\end{equation}

which takes an exponential form.

\paragraph{Expectation Maximization. }

Using Equation~\eqref{eq:joint_density}, and given all $\omega_{i}$, the log-likelihood objective over $\boldsymbol{\theta} \;=\;\bigl(\{\,r_j\}_{j=1}^K,\;\{\,\beta_l\}_{l=1}^M\bigr),$  simplifies to the following form upto additive constant that do not depend on $\bm{\theta}$ (more details in Appendix~\ref{sec:ll_detailed}).
\begin{align}
\ell(\boldsymbol{\theta}) 
&= \sum_{(w_i, l_i, s_i) \in \mathcal{D}}
\Bigl[ \tfrac{1}{2}\,\beta_{s_i}\,(r_{w_i} - r_{l_i}) \Bigr. \notag\\
&\quad \Bigl. - \tfrac{1}{2}\,\omega_i\,\beta_{s_i}^2\,(r_{w_i} - r_{l_i})^2 \Bigr]
\end{align}

Hence, this gives an  EM algorithm that iteratively estimates the latent variables $\omega_i$ followed by log-likelihood maximization steps. 
Algorithm \ref{alg:em_am} describes this EM Algorithm. The Expectation step (E-step) and Maximization step (M-step) are as follows (details in Appendix \ref{sec:em_detailed}):
\subsection{E-step}
At iteration \(t\), given the current estimates $\boldsymbol{\theta}^{(t)}$ we have 
\begin{equation}
\eta_i^{(t)} 
\;=\;
\beta_{\,s_i}^{(t)}\,\bigl(r_{\,w_i}^{(t)} - r_{\,l_i}^{(t)}\bigr),
\quad \forall i \in [N]
\label{eq:eta_i}
\end{equation}
Given $\eta_i^{(t)}$, the Pólya–Gamma random variable
$\omega_i$ follows the distribution
\(\mathrm{PG}(1,\;\eta_i^{(t)})\).
Hence, we can evaluate the expectation of $\omega_i$ given the current parameters. Let $\kappa_i^{(t)} = \mathbf{E}[\omega_i ~ |~\eta_i^{(t)} ]$. We show that 
(Appendix \ref{subsec:e_step_detailed})
\begin{equation}
\kappa_i^{(t)} 
= 
\begin{cases}
\frac{1}{2\,\eta_i^{(t)}} \,\tanh\left(\frac{\eta_i^{(t)}}{2}\right),
& \eta_i^{(t)} \neq 0 \\[8pt]
\frac{1}{4},
& \eta_i^{(t)} = 0
\end{cases}
\label{eq:kappas}
\end{equation}


This gives the formulation for the $Q$ function:
\begin{align}
Q\bigl(\boldsymbol{\theta} \mid \boldsymbol{\theta}^{(t)}\bigr)
&= \sum_{(w_i, l_i, s_i) \in \mathcal{D}}
\Bigl[ \tfrac{1}{2}\,\beta_{s_i}\,(r_{w_i} - r_{l_i}) \Bigr. \notag\\
&\quad \Bigl. - \tfrac{1}{2}\,\kappa_i^{(t)}\,\beta_{s_i}^2\,(r_{w_i} - r_{l_i})^2 \Bigr]
\label{eq:Qfunc_defn}
\end{align}



\begin{algorithm}
\caption{BoRaEM Algorithm}
\label{alg:em_am}
\begin{algorithmic}[1]
    \Require Dataset of pairwise comparisons $\mathcal{D}$
    \Ensure Parameter estimates $\boldsymbol{\theta} = (\boldsymbol{\beta}, \boldsymbol{r})$

    \State Initialize $\boldsymbol{\theta}^{(0)}$
    \State $t \gets 0$
    \Repeat
        \State \textbf{E-step:} 
        \State \quad $Q(\boldsymbol{\theta} \mid \boldsymbol{\theta}^{(t)}) \gets 
        \mathbb{E}_{\omega \mid \mathcal{D}, \boldsymbol{\theta}^{(t)}} [ \ell(\boldsymbol{\theta}; \mathcal{D}, \omega) ]$ \Comment{See \eqref{eq:Qfunc_defn}}
        
        \State \textbf{M-step (using AM):} 
        \State \quad Initialize $(\boldsymbol{\beta}^{(0)}, \boldsymbol{r}^{(0)}) \gets \theta^{(t)}$, $k \leftarrow 0$
        \quad \Repeat
            \quad\State $\boldsymbol{\beta}^{(k+1)} \gets \arg\max_{\boldsymbol{\beta}} Q(\boldsymbol{\beta}, \boldsymbol{r}^{(k)}\mid \boldsymbol{\theta}^{(t)})$
            \label{state:beta_update}
            \quad\State $\boldsymbol{r}^{(k+1)} \gets \arg\max_{\boldsymbol{r}} Q(\boldsymbol{\beta}^{(k+1)}, \boldsymbol{r} \mid \boldsymbol{\theta}^{(t)})$
            \label{state:r_u1}
            \quad\State  $\boldsymbol{r}^{(k+1)} \gets  \boldsymbol{r}^{(k+1)}  - \frac{\boldsymbol{1}^T\boldsymbol{r}^{(k+1)}}{n}$
            \label{state:r_u2}
            \State $k \gets k+1$
        \quad\Until convergence
        \State $\boldsymbol{\theta}^{(t+1)} \gets (\boldsymbol{\beta}^{(k-1)}, \boldsymbol{r}^{(k-1)})$
        \State $t \gets t + 1$
    \Until convergence
    \State \Return $\boldsymbol{\theta}^{(t+1)}$
    
\end{algorithmic}
\end{algorithm}

\subsection{M-step}
\label{subsec:m_step}
In order to maximize the $Q\bigl(\boldsymbol{\theta} \mid \boldsymbol{\theta}^{(t)}\bigr)$ function in terms of $\boldsymbol{\theta} =  (\boldsymbol{\beta}, \boldsymbol{r})$, we use the Alternating Maximization (AM) algorithm. 
We first initialize $\boldsymbol{\theta}$ to $\boldsymbol{\theta}^{(t)}$. We update $\boldsymbol{\beta}$ corresponding to Step \ref{state:beta_update} of Algorithm \ref{alg:em_am} as follows (Appendix \ref{subsec:m_step_detailed}) (Let $d_i = r_{w_i} - r_{l_i}$ denote the reward difference between the preferred item $w_i$ and the losing item $l_i$, evaluated at the current item reward estimates within the inner iterations of the M-step.)

\begin{equation}
\beta_s 
=\;
\frac{\tfrac1{2}\sum_{\,i\in I_s} d_i}{\sum_{\,i\in I_s} \kappa_i^{(t)}\,d_i^2}
\;=\;
\frac{\sum_{\,i\in I_s} d_i}{2 \,\sum_{\,i\in I_s} \kappa_i^{(t)}\,d_i^2}
\label{eq:betak_update}
\end{equation}


where
\(
I_s \;=\;\bigl\{\,i \in [N]: s_i = s\bigr\}
\)
as the index set of comparisons for worker \(s\). 

The reward vector $\mathbf{r}$ can be obtained uniquely by solving the linear system \(
\mathbf{H} \mathbf{r} = \mathbf{b},\) and imposing the constraint $\sum_{j=1}^{K} r_j = 0$,
where the matrix $\mathbf{H}$ is defined as:
\begin{align*}
H_{jj} &= \sum_{i \in W_j} \beta_{s_i}^2 \kappa_i^{(t)} + \sum_{i \in L_j} \beta_{s_i}^2 \kappa_i^{(t)},\\
H_{jl} &= \sum_{i \in \mathcal{D}_{jl}}
- \beta_{s_i}^2 \kappa_i^{(t)}, \quad j \neq l,
\end{align*}
and, the vector $\mathbf{b}$ is defined as:
\[
b_j = \sum_{i \in W_j} \frac{1}{2} \beta_{s_i} + \sum_{i \in L_j} \left(-\frac{1}{2} \beta_{s_i}\right),
\]
where $W_j$ and $L_j$ denote the sets of comparisons in which item $j$ won and lost, respectively, and $\mathcal{D}_{jl}$ is the set of comparisons between items $j$ and $l$. This corresponds to Steps \ref{state:r_u1} and \ref{state:r_u2} of Algorithm \ref{alg:em_am}. The linear system has a unique solution (Appendix \ref{subsec:m_step_detailed}) and can be solved efficiently using conjugate gradient descent. The implementation details are provided in Appendix \ref{app:pgem_implementation}.
\section{Convergence Analysis}
\label{sec:theoretical_guarantees}
In this section we show that the M-step objective can be written as a matrix sensing problem with a rank-1 parameter matrix $X = \betav \rv^\top$. This allows us to leverage standard results from low-rank matrix recovery to analyze the geometry of the Q-function and the behavior of EM.


Formally, Equation~\eqref{eq:Qfunc_defn} can be stated in the following form.
\[
\min_{\betav, \rv}
\|\mathcal{A}(\boldsymbol{\beta}\rv^T) - \mathbf{u}\|_2^2
    \,,
\]
where $\mathbf{u} \in \mathbb{R}^N$ and the linear operator $\opA \colon \reals^{M\times K} \to \reals^N$ as $\opA(X) = (\langle A_1, X\rangle, \dots, \langle A_N, X\rangle)^\top$ with matrices $A_i = 2c_i \alpha \ \bm{e_{s_i}} (\bm{e_{w_i}} - \bm{e_{l_i}})^\top$ 
where $\alpha = \sqrt{MK/2N}$,  $c_i^2 = \kappa_i^{(t)}$ for iteration t, $e_x$ denotes a unit vector having value as $1$ at index $x$. More details are in Appendix \ref{sec:convergence_analysis}.

We assume that the worker index and the two item indices (denoted by random variables $B, W,$ and $L$, respectively) are independent and uniformly distributed, $B \sim \text{Unif}\{ 1,2,\ldots,M\}$ and $W,L \sim \text{Unif}\{1,2,\ldots,K\}$. The comparison outcome is then generated according to \ref{eq:boltz}. Further, we assume that $\beta \in [-1,1]$ and $r \in [-R,R]$ for some constant $R>0$.


Under mild assumptions detailed in Appendix~\ref{sec:convergence_analysis}, the linear operator $\mathcal{A}$ satisfies the Restricted Isometry Property, according to the following definition \citep{candes2010tight}.

\begin{definition}[Restricted Isometry Property (RIP)]
A linear operator $\mathcal{A} : \mathbb{R}^{M \times K} \rightarrow \mathbb{R}^{N}$ with $(\mathcal{A}(X))_i = \langle A_i, X\rangle$ satisfies the restricted isometry property on rank-$r$ matrices, with parameter $\delta_r$, if the following set of inequalities hold for all rank-$r$ matrices $X$:
\[
(1-\delta_r)\|X\|_F^2 \le \|\mathcal{A}(X)\|_2^2 \le (1+\delta_r) \|X\|_F^2.
\]
\end{definition}
This gives us the following lemma.
\begin{restatable}{lemma}{rip}
\label{thm:rip_main}
Define matrix operators $A_i \in \reals^{M \times K}$ as $A_i = 2c_i \alpha \ \bm{e_{s_i}} (\bm{e_{w_i}} - \bm{e_{l_i}})^\top$, where $\alpha = \sqrt{MK/2N}$. Let $X \in \reals^{M \times K}$ be any matrix of arbitrary rank $r$ with $X\bm{1} = \bm{0}$, where $\bm{1}$ and $\bm{0}$ denote the all-ones and all-zeros vectors, respectively. Define the linear operator $\opA \colon \reals^{M\times K} \to \reals^N$ as $\opA(X) = (\langle A_1, X\rangle, \dots, \langle A_N, X\rangle)^\top$. Then, for any $\eta > 0$, there exists $\mu > 0$ and $\delta>0$ such that with probability at least $1-\eta$:
\[
(1-\mu)\norm{X}{F}^2 \le \|\opA(X)\|_2^2 \le (1+\mu) \norm{X}{F}^2 \, ,
\]
provided that 
$$N\geq C\;\frac{MK}{\delta^2}\ln\left(\frac{2}{\eta}\right),$$ where $C > 0$ is some constant.
\end{restatable}
Then, invoking Remark~3 from \cite{park2016nonsquarematrixsensingspurious}, we can state the following theorem.
\begin{restatable}{theorem}{convergence}
\label{thm:convergence}
    Let $\{ \betav^{(t)}, \rv^{(t)} \}_{t \geq 0}$ be the sequence of iterates generated by the EM algorithm. Suppose the M-step at each iteration $t$ solves the optimization problem:
\[
(\betav^{(t+1)}, \rv^{(t+1)}) = \arg \min_{\betav,\rv} \|\mathcal{A}(\betav \rv^\top) - \mathbf{u}^{(t)}\|_2^2
\]
having global minimum $(\bm{\theta}^*)^{{(t+1)}}$, where $u_i^{(t)} = \frac{\alpha}{c_i}$. If $R = 0.1$, the operator $\mathcal{A}$ satisfies Restricted Isometry Property with $\mu_2 = \mu_{4} = 0.005$, and every local minimum $\bm{\theta}^{(t+1)} = \left(\betav^{(t+1)}; \rv^{(t+1)}\right)$ satisfies:
\[
\text{DIST}(\bm{\theta}^{(t+1)}, (\bm{\theta}^*)^{{(t+1)}}) \leq \frac{1250}{3\sigma_1(X^*_t)} \cdot \|\mathcal{A}(X^*_t - X_{1,t}^*)\|_2
\]
where $X^*_t \bm{1} = \bm{0}$ and $X_{1,t}^*$ is the best rank-1 approximation of the true high-rank matrix $X^*_t$.
\end{restatable}
Therefore, each EM update computes a solution 
whose distance to the global optimum of the corresponding M-step objective is bounded by the intrinsic rank-one approximation error of the corresponding M-step objective.
Thus, the estimation error at each iteration is controlled by $\|\mathcal{A}(X^*_t - X_{1,t}^*)\|_2$. At iteration $t$, if $X_t^*$ is exactly rank-1, the algorithm recovers the global minimum of the function $Q(\bm{\theta} \mid \bm{\theta}^t)$.

Following Proposition 2.7.1 from \citet{Bertsekas01031997}, we can say the alternating minimization (AM) procedure used to solve the M-step converges to a stationary point, since our $Q(\bm{\theta} \mid \bm{\theta}^t)$ function is differentiable and marginally convex. 
In our case, in each iteration of AM, we solve for $\bm{\beta}$ and $\mathbf{r}$ exactly. Therefore, Theorem~4.3 from \citet{park2016nonsquarematrixsensingspurious} implies that alternating minimization escapes saddle points and converges to a local minimum with high probability, which, by Theorem~\ref{thm:convergence}, is close to the global minimum.


Finally, since the M-step at each iteration is solved using AM, which is a monotonic procedure \citep{prateekjain}, each EM iteration monotonically increases the likelihood. Standard convergence results for EM therefore imply that the sequence of iterates converges to a stationary point of the likelihood.


%



\section{EXPERIMENTAL ANALYSIS}
\label{sec:experimental_analysis}

\subsection{Evaluation Metrics}
\label{sec:eval_metrics}
To quantitatively evaluate the performance of the proposed methods, we employ Kendall's Tau ($\tau$), Accuracy ($acc$), and Weighted Accuracy ($wacc$). These metrics focus on the relative ordering of items rather than absolute value errors. 
Kendall's Tau is a standardized rank correlation that is robust to outliers, but it treats all pairwise swaps equally. Accuracy measures the probability of correctly ordering item pairs and is easy to interpret, though it ignores the magnitude of score differences. Weighted Accuracy penalizes larger ranking errors more strongly, but its value depends on the distribution of the scores.
The definitions, including their respective advantages and limitations, are summarized in Table~\ref{tab:metrics} in Appendix \ref{app:eval_metrics}. we evaluate all the representative crowdsourcing methods discussed in Section~\ref{sec:related_work}. More details are given in \ref{app:code_methods}.



\begin{figure*}[h]
    \centering
        \begin{subfigure}[t]{0.8\linewidth}
        \centering
        \includegraphics[width=\linewidth]{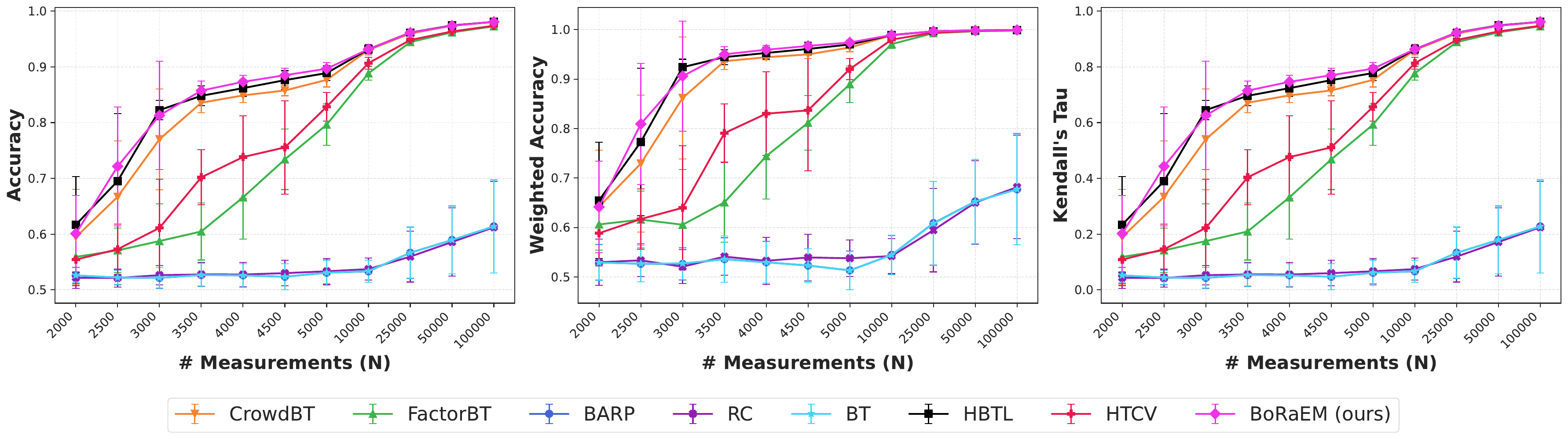}
        \subcaption{Varying \(N\). \(K=100,\,M=500\).}
        \label{fig:sub_vary_N_-1_1}
    \end{subfigure}\hfill
    \begin{subfigure}[t]{0.8\linewidth}
        \centering
        \includegraphics[width=\linewidth]{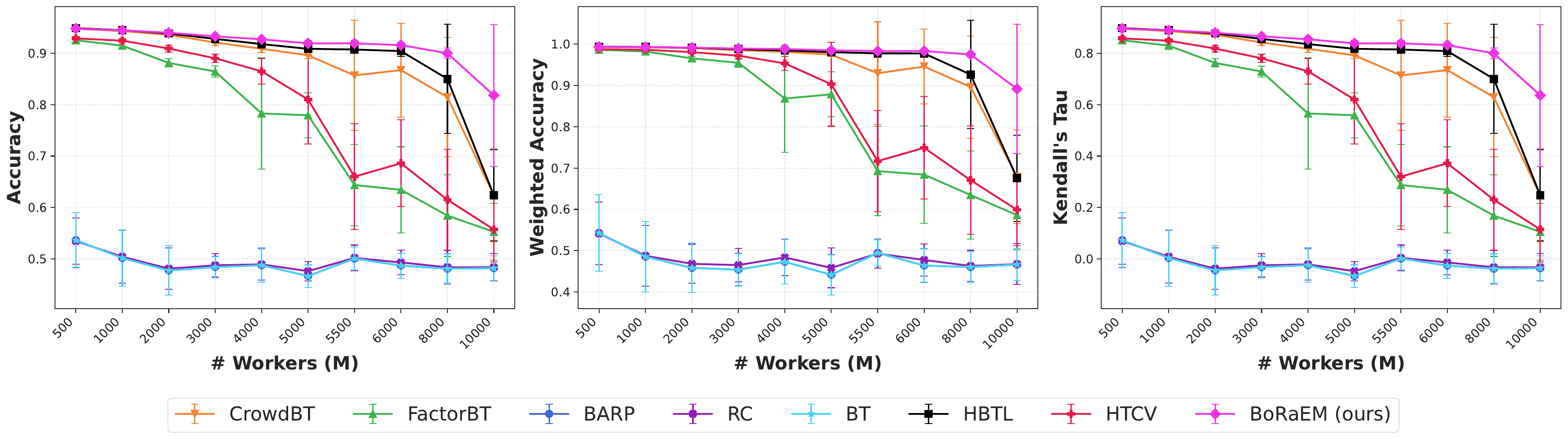}
        \subcaption{Varying \(M\). \(K=200,\,N=25000\).}
        \label{fig:sub_vary_M_-1_1}
    \end{subfigure}\hfill
    \begin{subfigure}[t]{0.8\linewidth}
        \centering
        \includegraphics[width=\linewidth]{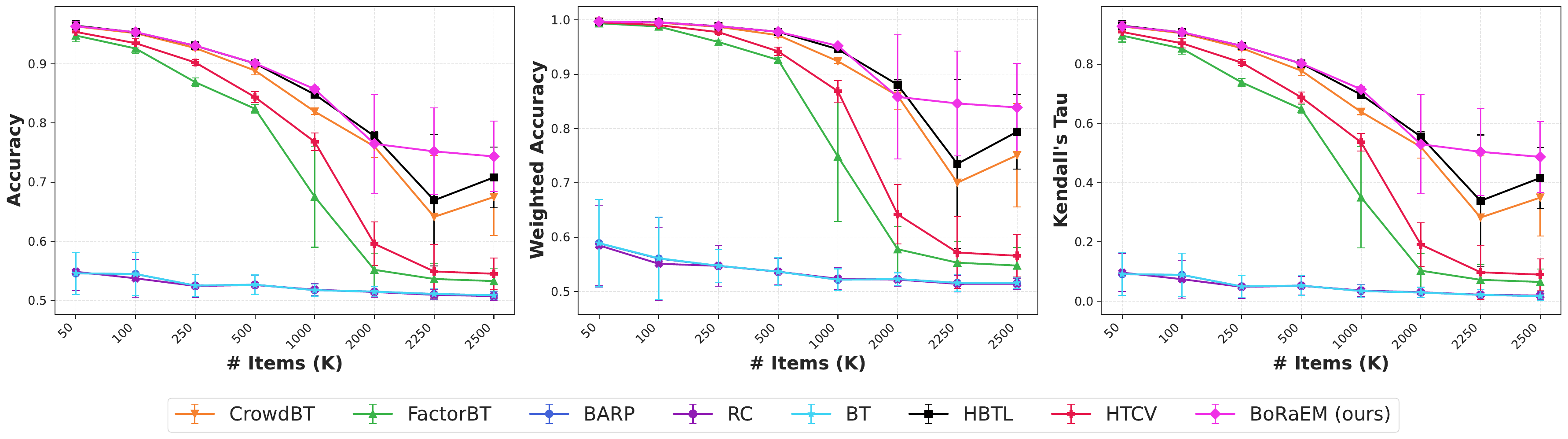}
        \subcaption{Varying \(K\). \(M=20000,\,N=1000\).}
        \label{fig:sub_vary_K_-1_1}
    \end{subfigure}
    \caption{Varying \(N\), \(K\), \(M\) individually when \(\beta \in [-1,1]\).}
    \label{fig:vary_all_-1_1}
\end{figure*}

\subsection{Synthetic Datasets}
\label{sec:synthetic_dataset_exp}

We will sample data (having $N$ measurements) from the distribution given in Eq.~\eqref{eq:extended_btl}, with $K$ (number of items) and $M$ (number of workers) to study their effect and we sample $\beta$ uniformly from $[-1, 1]$. 
We use 10 random seeds for data sampling.
Let $k$ denote 1000 for specifying number of measurements, workers, or items.
To study the effect of each variable, we vary $N$, $K$, $M$ individually keeping the other two variables constant as follows:




\begin{enumerate}
    \item \textbf{Varying N, Keeping K and M constant:}
We conduct an experiment with fixed values of the number of items $K = 100$ and the number of workers $M = 500$, while varying the number of measurements $N \in \{2k, \allowbreak 2.5k, \allowbreak 3k, \allowbreak 3.5k, \allowbreak 4k, \allowbreak 4.5k, \allowbreak 5k, \allowbreak 10k, \allowbreak 25k, \allowbreak 50k, \allowbreak 100k\}$ to examine the effect of measurement volume on various techniques.
As seen in Figure \ref{fig:sub_vary_N_-1_1}, irrespective of number of measurements the performance of methods not modeling the annotator competence remains low, whereas the methods that model annotator competence while predicting rewards i.e., BoRaEM, HBTL and CrowdBT remains high. BoRaEM seems to be robust than CrowdBT in sparse measurements setting (i.e, for measurements less than 4000).

\item \textbf{Varying M, Keeping N and K constant:}
We vary the number of workers $M \in \{500, \allowbreak 1k, \allowbreak 2k, \allowbreak 3k, \allowbreak 4k, \allowbreak 5k, \allowbreak 5.5k, \allowbreak 6k, \allowbreak 8k\}$ while fixing the number of items at $K = 200$ and the number of measurements at $N = 25k$, to analyze the impact of worker scale.
As seen in Figure \ref{fig:sub_vary_M_-1_1}, all methods suffers as the number of workers are increased (or the comparisons per worker get decreased). The performance decay of CrowdBT is seen to be much higher than BoRaEM and BoRaEM is quite stable.

\item \textbf{Varying K, Keeping N and M constant:}
We vary the number of items $K \in \{50, \allowbreak 100, \allowbreak 250, \allowbreak 500, \allowbreak 1000, \allowbreak 1250, \allowbreak 1500, \allowbreak 1750, \allowbreak 2250, \allowbreak 2500\}$ while fixing the number of workers at $M = 20k$ and the number of measurements at $N = 1k$, in order to study the effect of item scale.
As seen in Figure \ref{fig:sub_vary_K_-1_1}, all methods suffers as the number of items are increased. The performance of BoRaEM, HBTL and CrowdBT decays after $K = 1000$, HBTL and BoRaEM seem to degrade slower than CrowdBT under larger N regimes.
    
\end{enumerate}


Methods such as \texttt{BT}, \texttt{BARP}, and \texttt{RC}, which do not explicitly model annotator competence or susceptibility to bias, exhibit consistently poor performance in the presence of adversarial workers. Appendix~\ref{app:syn_data_beta_0_1} presents results for the regime $\beta \in [0, 1]$, which contains no adversarial annotators. 
In this setting, these methods perform comparatively better 
due to
absence adversaries.
In contrast, \texttt{BoRaEM}, \texttt{HBTL}, and \texttt{CrowdBT}, which explicitly model annotator competence, achieve superior performance in denser regimes where sufficient observations are available to reliably estimate the additional parameters. However, their performance exhibits a modest decline in sparse regimes, where limited data constrains accurate parameter estimation and increases susceptibility to estimation noise. Hardware details are in Appendix \ref{sec:hardware}.


\subsection{Real Datasets}

\begin{table*}[t]
  \caption{Summary statistics of the evaluation datasets.}
  \label{tab:dataset_stats}
  \centering
  \resizebox{0.9\linewidth}{!}{
    \begin{tabular}{@{}>{\raggedright\arraybackslash}p{0.15\linewidth}p{0.4\linewidth}p{0.4\linewidth}@{}}
        \toprule
        \textbf{Attribute} & \textbf{FaceAge Dataset} & \textbf{Reading Difficulty or Passage Dataset} \\
        \midrule
        Items ($K$) & 9,150 & 472 \\
        Workers ($M$) & 4,091 & 624 \\
        Comparisons ($N$) & 250,249 & 11,763 \\
        \addlinespace
        Description & Comparison of two face images to identify the older individual. & Pairwise difficulty assessment of reading passages. \\
        \addlinespace
        Ground Truth & Chronological age (years) & Standardized reading levels \\
        \addlinespace
        Spammer Handling & Removal based on performance on ``gold" instances (large age gaps). & Not specified in original study. \\
        \bottomrule
    \end{tabular}
  }
\end{table*}


\begin{table*}[htbp]
\centering
\caption{Benchmark results on FaceAge and Passage datasets. Accuracy and Weighted Accuracy are reported as percentages, while Kendall's Tau is reported as a decimal. Best overall results are shown in \textbf{bold}.}
\label{tab:reorganized_results}
\resizebox{0.9\linewidth}{!}{
\begin{tabular}{@{}lcccc@{}}
\toprule
\textbf{Method} & \textbf{Accuracy (\%)} & \textbf{Weighted Accuracy (\%)} & \textbf{Kendall's Tau} & \textbf{Runtime (s)} \\
\midrule
\rowcolor[gray]{0.95} \multicolumn{5}{l}{\textit{FaceAge Dataset}} \\
Simple \texttt{BT} & 79.01 & 87.20 & 0.5754 & $0.89 \pm 0.05$ \\
RC                 & 78.04 & 86.44 & 0.5582 & $0.07 \pm 0.03$ \\
FactorBT           & 79.16 & 87.30 & 0.5785 & $169.41 \pm 1.16$ \\
BARP               & 79.08 & 87.26 & 0.5769 & $51.08 \pm 2.85$ \\
CrowdBT            & 79.17 & 87.31 & 0.5787 & $138.16 \pm 3.86$ \\
HTCV               & 79.17 & 87.33 & 0.5786 & $41.31 \pm 1.41$ \\
HBTL               & 79.20 & 87.35 & 0.5793 & $83.10 \pm 3.32$ \\
BoRaEM (Ours)      & \textbf{79.21} & \textbf{87.36} & \textbf{0.5795} & $6.55 \pm 0.43$ \\
\midrule
\rowcolor[gray]{0.95} \multicolumn{5}{l}{\textit{Passage Dataset}} \\
Simple \texttt{BT} & 68.04 & 74.33 & 0.3407 & $1.39 \pm 0.13$ \\
RC                 & 65.22 & 71.30 & 0.2892 & $0.02 \pm 0.04$ \\
FactorBT           & 69.47 & 75.58 & 0.3676 & $100.85 \pm 2.82$ \\
BARP               & 68.19 & 74.51 & 0.3435 & $32.05 \pm 2.38$ \\
CrowdBT            & \textbf{70.02} & \textbf{75.96} & \textbf{0.3779} & $84.25 \pm 1.84$ \\
HTCV               & 69.30 & 75.30 & 0.3643 & $46.19 \pm 1.83$ \\
HBTL               & 69.53 & 75.46 & 0.3688 & $77.72 \pm 2.85$ \\
BoRaEM (Ours)      & 69.59 & 75.56 & 0.3698 & $8.28 \pm 0.71$ \\
\bottomrule
\end{tabular}
}
\end{table*}

\begin{figure*}[htbp]
    \centering
    \includegraphics[width=\linewidth]{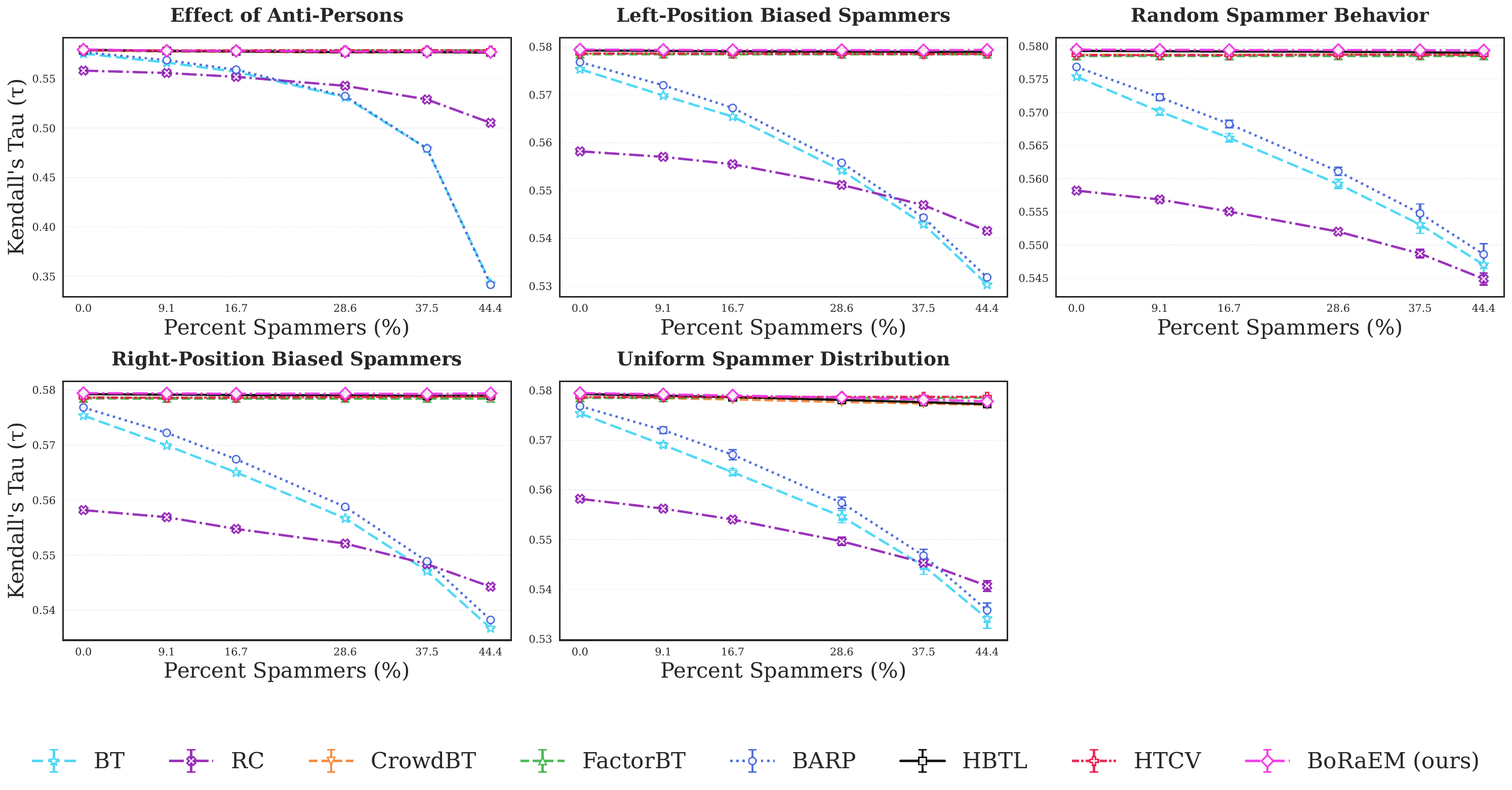}
    \caption{Kendall's Tau (\textit{$\tau$}) of different ranking methods under varying proportions of spammers for \textsc{FaceAge} dataset.}
    \label{fig:spammer_age_addition_tau}
\end{figure*}

For evaluating crowdsourcing methods on real-world data, we use two datasets: the FaceAge dataset (IMDB-WIKI-SbS)\footnote{\url{http://www-personal.umich.edu/~kevynct/datasets/wsdm_rankagg_2013_readability_crowdflower_data.csv}} \citep{pavlichenko2021imdb} and the Reading Difficulty (Passage) dataset\footnote{\url{https://github.com/Toloka/IMDB-WIKI-SbS/tree/main}} \citep{chen2013pairwise}. The dataset statistics are summarized in Table~\ref{tab:dataset_stats}. 

The FaceAge dataset~\citep{pavlichenko2021imdb}  is specifically tailored for large scale evaluation of various crowdsourcing methods and is a widely adopted benchmark for crowdsourced pairwise comparison experiments~\citep{ferrara2024bias}. On the FaceAge dataset, \texttt{BoRaEM} achieves the best performance, followed by \texttt{HBTL} and \texttt{CrowdBT}. 
The Passage dataset is comparatively much smaller.
On the \textsc{Passage} dataset, we observe that all methods achieve relatively low performance, with Kendall's tau values ranging from 0.29 to 0.38. 
This can be attributed to the small size of the dataset and the inherent difficulty of modeling preferences in reading difficulty, which is a cognitively complex and subjective task. On this dataset, \texttt{CrowdBT} achieves the highest performance, followed by \texttt{HBTL} and \texttt{BoRaEM}.

Table \ref{tab:reorganized_results} presents the results on the FaceAge and Passage datasets.
\texttt{BoRaEM}, \texttt{HBTL}, \texttt{HTCV}, \texttt{FactorBT} and \texttt{CrowdBT} remain competitive across both datasets. These results highlight the importance of jointly modeling worker competence and item rewards. 

\subsection{Spammer Addition Analysis}
\label{sec:spammer_analysis}
We evaluate robustness under adversarial conditions by injecting synthetic spammers commonly studied in the crowdsourcing literature. Specifically, we consider four archetypes: (i) \textit{random choosers}, who respond uniformly at random without regard to item quality \citep{vuurens2011much}; (ii) \textit{position-biased} spammers, who systematically favor one side of the comparison (e.g., always selecting the left or right option), also referred to as uniform spammers in prior work \citep{day1969position, xu2016false, vuurens2011much}; (iii) \textit{malicious adversaries}, who intentionally provide inverted or misleading responses to degrade aggregation quality \citep{kleindessner2018crowdsourcing, steinhardt2016avoiding}; and (iv) a \textit{combined} setting, where all aforementioned spammer types are present in equal proportion, reflecting a more realistic and heterogeneous adversarial environment. For fairness, each injected spammer is assigned the same expected number of comparisons as an original worker, ensuring performance differences arise from behavioral characteristics rather than annotation volume.

\subsubsection{Results}
We mainly focus on the FaceAge dataset, as it is a widely adopted benchmark for crowdsourcing \citep{pavlichenko2021imdb, ferrara2024bias}. Results on the Passage dataset are provided in Appendix \ref{app:spammer_passage}. 
Figure \ref{fig:spammer_age_addition_tau} presents the performance of different crowdsourcing methods in terms of Kendall’s tau ($\tau$) as the proportion of various types of spammers in the FaceAge dataset is varied. 
The results are averaged over 10 random seeds for generating spammers and 10 random seeds for random initializations. 
Similar trends are observed in terms of accuracy ($acc$) and weighted accuracy ($wacc$) (see Appendix \ref{app:spammer_faceage_more}).

We add spammers following the procedure of \citet{bugakova2019aggregation}. 
Specifically, if there are originally $M$ annotators, we add $0.1M$, $0.2M$, $0.4M$, $0.6M$, and $0.8M$ spammers according to the specified proportions. 
As the spammer proportion increases from 0\% to 44.4\%, \texttt{BoRaEM}, \texttt{HBTL}, \texttt{HTCV} consistently perform the best, followed by \texttt{CrowdBT} and \texttt{FactorBT}, while \texttt{BARP}, \texttt{BT}, and \texttt{RC} exhibit substantial performance drops. 
These results highlight the importance of explicitly modeling annotator competencies in crowdsourcing. Further, we have performed ablations over intialization and hyperparameters, and the results are detailed in Appendix \ref{sec:pgem_ablation}.

\section{CONCLUSION}
\label{sec:conclusion}
In this work, we addressed the problem of rank aggregation from pairwise comparisons in the presence of workers with varying competencies. We adapt the Boltzmann-rational model for modeling each pairwise comparison. We then use the Polya-Gamma random variable augmentation to simplify the logistic function to an exponential form enabling a simple and efficient EM formulation. We achieve tractable inference with theoretical convergence guarantees by treating this problem as a matrix sensing problem. 
Extensive experiments on both synthetic and real-world datasets demonstrate the advantages of
using our algorithm over several baselines and the algorithm remains robust even in the presence of a large fraction of spammers, highlighting its effectiveness for reliable rank aggregation in crowdsourced settings. 

Future work includes designing models for robust RLHF that can jointly predict item feature importance vectors while effectively identifying and mitigating the influence of spammers. 

\begin{contributions}
Arpit Agarwal conceived the idea of modeling annotator competence using the Boltzmann Rational Model, identified its connection to the matrix sensing problem, and led the manuscript preparation.
Kaustubh Shivshankar Shejole proposed the use of Pólya-Gamma augmentation, developed the method, conducted empirical studies, and generated the experimental visualizations.
Tanish Agarwal developed the theoretical foundations of the proposed approach under the guidance of Arpit Agarwal and Avishek Ghosh.
\end{contributions}

\begin{acknowledgements}
We would like to thank the Organizing Committee, Area Chairs, Senior Area Chairs, Meta-Reviewer, and the anonymous reviewers of UAI 2026 for their valuable feedback, constructive suggestions, and insightful comments, which helped improve the quality and presentation of this work.

The first author gratefully acknowledges Kush Mangukiya for contributions during the early stages of this research and Pratham Shah, Manish Kumar for their continuous support and encouragement.
\end{acknowledgements}

\sloppy
\bibliography{uai2026-template}

\newpage

\onecolumn

\title{Finding the Signal in the Spam (Supplementary Material)
}
\maketitle

\appendix

\section{More Details on BoRaEM Derivation}
\label{app:pgem_derivation}

\subsection{Pólya-Gamma Random Variables}
\label{sec:appendix_pg_definition}

Following \citet{polson2013bayesian}, if $\omega$ is a Pólya-Gamma distributed random variable with parameters $b > 0$ and $c \in \mathbb{R}$, denoted $\omega \sim \mathrm{PG}(b, c)$, then it is equal in distribution to an infinite weighted sum of gamma random variables:
\begin{equation*}
    \omega  \stackrel{d}{=}  
    \frac{1}{2\pi^2}\,
    \sum_{k=1}^\infty
    \frac{g_k}
    {\left(k - \frac12\right)^2 + \frac{c^2}{4\pi^2}}\,, 
    \quad
    g_k \sim \mathrm{Gamma}(b,1).
\end{equation*}
Here, ``\(\stackrel{d}{=}\)'' denotes equality in distribution, and $g_k \sim \text{Gamma}(b, 1)$ are independent Gamma random variables. Note that this is not the density function. Instead, equality in distribution means that the PG random variable on the left-hand side has the same cumulative distribution function as the random variable on the right-hand side.

\citet{polson2013bayesian} showed that all finite moments of $\omega$ admit closed form expressions; in particular,
\begin{equation}
    \label{eq:expectation_polya_gamma_variable}
    \E[\omega] =
    \frac{b}{2c}\,\tanh\left(\frac{c}{2}\right).
\end{equation}
They further proved the following two identities for $\omega \sim \mathrm{PG}(b,0)$, which are central to Pólya-Gamma data augmentation for logistic models:
\begin{enumerate}
    \item For any real $\psi$, one has
        \begin{equation}
            \label{eq:appendix_pg_identity}
            \frac{(e^{\psi})^a}{(1 + e^{\psi})^b} = 2^{-b} e^{\kappa \psi} \int_0^{\infty} e^{-\omega \psi^2 / 2} \, p(\omega\mid b,0) \, d\omega,
        \end{equation}
        where $\kappa = a - \frac{b}{2}$, and $p(\omega\mid b,0)$ is the density of a $\mathrm{PG}(b,0)$ random variable.

    \item The conditional density of \(\omega\) given \(\psi\) is again Pólya–Gamma with parameters $b$ and $\psi$, that is,
        \begin{equation}
            (\omega\mid \psi) \sim \mathrm{PG}(b,\psi).
        \end{equation}
\end{enumerate}

\subsection{Log Likelihood Derivation}
\label{sec:ll_detailed}

Let us consider the problem formulation given in Section \ref{sec:problem_formulation}, where there is a set of $K$ items, and one observes outcomes of pairwise comparisons
between these items, elicited from a set of $M$ crowdworkers.
Specifically, we are given a dataset $\mathcal{D} = \{(w_i, l_i, s_i)\}_{i=1}^N$ where there are $N$ pairwise comparisons; the $i$-th comparison is performed by worker $s_i \in [M]$ where $w_i \in [K]$ is the winning item and $l_i \in [K]$ is the losing item such that $l_i$ is not equal to $w_i$.

Consider Equation~\eqref{eq:appendix_pg_identity}. Let $ \eta_{i} = \beta_{s_i}(r_{w_i} - r_{l_i})$. For \(a=1\), \(b=1\), and \(\psi = \eta_{i}\), we have \(\kappa = \frac12\) and \(p(\omega_{i})=\mathrm{PG}(\omega_{i}\mid 1,0)\), which allows us to write

\begin{equation}
\frac{1}{1 + e^{-\eta_{i}}}
=
\frac{e^{\eta_i}}{1 + e^{\eta_{i}}}
=
\frac{1}{2}e^{\eta_{i}/2}
\int_{0}^{\infty}
\exp\left(-\frac{\omega_{i}\,\eta_{i}^2}{2}\right)\,p(\omega_{i})\,d\omega_{i}
\label{eq:pg_identity_extension}
\end{equation}

Let $\bm{\theta} = \{\textbf{r}, \bm{\beta}\}$.
Therefore, referring to Equation~\eqref{eq:eta_i} and ~\eqref{eq:pg_identity_extension}, we can write each term of the likelihood as
\begin{align}
    p(w_i, l_i, s_i | \bm{\theta})
    &=
     \frac{1}{1 + \exp{\left\{
        -\beta_{s_i}(r_{w_i} - r_{l_i})
     \right\}}}\\
     &=
    \frac{1}{2}e^{\eta_i/2}
    \int_{0}^{\infty}
    \exp\left(-\frac{\omega_{i}\eta_i^2}{2}\right)\,p(\omega_{i})\,d\omega_{i}\notag\\
    &=
    \frac{1}{2}
    \int_{0}^{\infty}
    \exp\left(\frac{\eta_i}{2} -\frac{\omega_{i}\eta_i^2}{2}\right) p(\omega_{i})\,d\omega_{i} \notag \\
\end{align}
which induces the augmented joint density
\begin{align}
    p(w_i, l_i, s_i, \omega_i | \bm{\theta})
    &\propto \exp\left(-\frac{\omega_{i}\eta_i^2}{2} + \frac{\eta_i}{2}\right) p(\omega_i)
    \notag,
\end{align}

Taking the logarithm and ignoring additive constants that do not depend on $\bm{\theta}$, we get the log-likelihood for the $i^{th}$ comparison as

\begin{align}
    \log(p(w_i, l_i, s_i, \omega_i | \bm{\theta}))
    &\propto
    \log\left[
\exp\left(\frac12\eta_{i} - \frac{\omega_{i}}{2}\,\eta_{i}^2\right)\,p(\omega_{i})\right]
+\text{const}
\\
&=
    -\frac{\omega_{i}\eta_i^2}{2} + \frac{\eta_i}{2}
    \notag,
    \\
    &= \frac{1}{2}\,\left( \beta_{s_i}(r_{w_i} - r_{l_i}) \right)
-\frac{\omega_{i}}{2}\,\left( \beta_{s_i}(r_{w_i} - r_{l_i}) \right)^2.
\end{align}

Summing over all $N$ comparisons in $\mathcal{D}$ gives the \textit{complete data} log-likelihood, given all $\omega_{i}$,
\begin{equation}
\ell(\bm{\theta}, \omega_i; \mathcal{D})
 = 
\frac12\sum_{i \in [N]}
\left[\,
\beta_{s_i}(r_{w_i} - r_{l_i})
- \omega_{i}\,\left( \beta_{s_i}(r_{w_i} - r_{l_i}) \right)^2
\right],
\end{equation}
up to additive constants independent of $\bm{\theta}$.





\subsection{Expectation-Maximization}
\label{sec:em_detailed}

\bigskip
\subsubsection{E-Step}
\label{subsec:e_step_detailed}
Given the current parameter estimate \(\theta^{(t)}\) and let $\boldsymbol{\omega}$ denote the set of all $w_i$ for $i=1$ to $N$ comparisons, we compute
\[
Q\left(\theta \mid \theta^{(t)}\right)
 = 
\mathbb{E}_{\,\omega \mid \mathcal{D},\,\theta^{(t)}}\left[\,
\ell(\bm{\theta}, \bm{\omega}; \mathcal{D})\right]
 = 
\frac1{2}\sum_{(w_i, l_i, s_i) \in \mathcal{D}}
\mathbb{E}_{\,\omega_i \mid \theta^{(t)}}\left[\,
\beta_{s_i}\,(r_{w_i}-r_{l_i})
 - \omega_i\,\beta_{s_i}^2\,(r_{w_i}-r_{l_i})^2
\right].
\]
Noting that the only term depending on $\omega_i$ is $\omega_i\,\beta_{s_i}^2\,(r_{w_i}-r_{l_i})^2$, we define
\[
\kappa_i^{(t)} 
 =  
\mathbb{E}\left[\omega_i \mid \theta^{(t)}\right].
\]
Then,
\begin{equation}
Q\left(\theta \mid \theta^{(t)}\right)
 = 
\frac1{2}\sum_{(w_i, l_i, s_i) \in \mathcal{D}}
\left[
\beta_{s_i}\,\left(r_{w_i} - r_{l_i}\right)
 - \kappa_i^{(t)}\,\beta_{s_i}^2\,\left(r_{w_i} - r_{l_i}\right)^2
\right].
\end{equation}

\bigskip

\paragraph{Computing \(\kappa_i^{(t)}\).} At iteration \(t\), we compute the conditional expectation 
\(\kappa_{i}^{(t)} = \mathbb{E}[\omega_{i} \mid \eta_{i}^{(t)}]\) for the latent 
variables \(\omega_i \sim \mathrm{PG}(1, \eta_i^{(t)})\), where 
\(\eta_i^{(t)} = \beta_{s_i}^{(t)} (r_{w_i}^{(t)} - r_{l_i}^{(t)})\). 
Invoking Eq.~\eqref{eq:expectation_polya_gamma_variable} yields:
\begin{equation}
\kappa_{i}^{(t)} = \mathbb{I}\left(\eta_{i}^{(t)} \neq 0\right) \frac{\tanh\left(\eta_{i}^{(t)}/2\right)}{2\eta_{i}^{(t)}} + \mathbb{I}\left(\eta_{i}^{(t)} = 0\right) \frac{1}{4}.
\end{equation}






\subsubsection{M-Step}

\label{subsec:m_step_detailed}
Note that while updating \boldsymbol{$\beta$},
\textbf{r} is kept fixed and vice-versa. 
\paragraph{Updating \textbf{$\beta$}}
\[
Q\left(\theta \mid \theta^{(t)}\right)
=\sum_{(w_i, l_i, s_i) \in \mathcal{D}}
\left[\,
\tfrac12\,\beta_{s_i}\,(r_{w_i}-r_{l_i})
 - \tfrac12\,\kappa_i^{(t)}\,\beta_{s_i}^2\,(r_{w_i}-r_{l_i})^2
\right],
\]
where \(\kappa_i^{(t)} = \mathbb{E}\left[\omega_i\mid \theta^{(t)}\right]\). We first maximize $Q$ with respect to each $\beta_s$. 
%
Let 
\[
I_s  = \left\{\,i : s_i = s\right\}
\]
denote the index set of comparisons done by worker $s$. Let $d_i = r_{w_i} - r_{l_i}$ denote the reward difference between the preferred item $w_i$ and the losing item $l_i$, evaluated at the current item reward estimates within the inner iterations of the M-step for all $i\in I_s$. Then, 
the fraction of \(Q\) depending on \(\beta_s\) is
\begin{equation}
Q_k\left(\beta_s\right)
 = 
\sum_{\,i\in I_s}
\left[
\tfrac12\,\beta_s\,d_i 
 - \tfrac12\,\kappa_i^{(t)}\,\beta_s^2\,d_i^2
\right].
\label{eq:Qk}
\end{equation}
For maximization, we require $\frac{\partial Q_k}{\partial \beta_s} = 0$, that is,
\begin{align}
\frac{\partial Q_k}{\partial \beta_s}
&=
\sum_{\,i\in I_s}
\left[
\tfrac12\,d_i 
 - \kappa_i^{(t)}\,\beta_s\,d_i^2
\right] = 0,
\label{eq:dQdBeta}
\end{align}
which on rearranging gives
\[
\tfrac12\sum_{\,i\in I_s} d_i 
 - \beta_s \sum_{\,i\in I_s} \kappa_i^{(t)}\,d_i^2 
 =  0
  \Longrightarrow  
\beta_s 
= 
\frac{\tfrac12\sum_{\,i\in I_s} d_i}{\sum_{\,i\in I_s} \kappa_i^{(t)}\,d_i^2}
 = 
\frac{\sum_{\,i\in I_s} d_i}{2\,\sum_{\,i\in I_s} \kappa_i^{(t)}\,d_i^2}.
\]

\paragraph{Updating \textbf{$r$}}

We now maximize $Q$ with respect to each $r_k$. 
Define the index sets
\[
W_j = \{\,i : w_i = j\}
\text{ and }
L_j = \{\,i : l_i = j\}
\]
where $W_j$ denotes the set of comparisons where item $j$ won, and $L_j$ denotes the set of comparisons where item $j$ lost. Then,
\begin{equation}
\frac{\partial Q}{\partial r_k}
 = 
 \left[
\sum_{\,i \in W_j}
\left[\,
\tfrac12\,\beta_{s_i}
 - \beta_{s_i}^2\,\kappa_i^{(t)}\,\left(r_k - r_{l_i}\right)
\right]
 + 
\sum_{\,i \in L_j}
\left[\,
-\tfrac12\,\beta_{s_i}
 + \beta_{s_i}^2\,\kappa_i^{(t)}\,\left(r_{w_i} - r_k\right)
\right]
\right]
= 0.
\label{eq:deriv_zero}
\end{equation}
which on rearranging gives
\[
r_k 
 =  \frac{A_k}{B_k}
\]
where
\begin{align*}
A_k
&= \sum_{\,i \in W_j} 
\left[\tfrac12\,\beta_{s_i} + \beta_{s_i}^2\,\kappa_i^{(t)}\,r_{l_i}\right]
 + 
\sum_{\,i \in L_j}
\left[-\tfrac12\,\beta_{s_i} + \beta_{s_i}^2\,\kappa_i^{(t)}\,r_{w_i}\right] \text{, and,}
\\
B_k
&= \sum_{\,i \in W_j}\beta_{s_i}^2\,\kappa_i^{(t)}
+ \sum_{\,i \in L_j}\beta_{s_i}^2\,\kappa_i^{(t)}.
\end{align*}

Observe that for fixed $\betav$, these update equations can be written as a linear system $H\rv = \mathbf{b}$ as follows.
\[
B_k r_k - \sum_{i \in W_j} \beta_{s_i}^2 \kappa_i r_{l_i}  - \sum_{i \in L_j} \beta_{s_i}^2 \kappa_i r_{w_i}=  \sum_{i \in W_j} \left(\frac{1}{2} \beta_{s_i}\right) + \sum_{i \in L_j}\left(-\frac{1}{2}\beta_{s_i} \right)
\]
The coefficient matrix $H$ and vector $\mathbf{b}$ are given by
\begin{equation*}
    H_{jl} =
    \begin{cases}
         \sum_{i \in \mathcal{D}_{jl}} - \beta_{s_i}^2 \kappa_i^{(t)}, \quad j\neq l\\
        B_j, \quad j=l\\
    \end{cases}
\end{equation*}
and
\[
b_j = \sum_{i \in W_j} \frac{1}{2}\beta_{s_i} + \sum_{i \in L_j} \left(-\frac{1}{2} \beta_{s_i}\right)
\]

By construction, $H$ is symmetric, $H_{jl} \le 0$ for $j\neq l$, and $H_{jj} = \sum_{l\neq j} |H_{jl}|$. Thus, $H$ is diagonally dominant with non-negative diagonal entries, and it is positive semidefinite.
We can further observe that $H\bm{1} = \bm{0}$. Now, suppose $\mathbf{x}$ satisfies $H\mathbf{x} = \bm{0}$ and let $j$ be such that $x_j = \max_i x_i$. Using the zero row sum property, we get that
\[
0 = (H\mathbf{x})_j = \sum_{l\neq j} H_{jl} (x_l - x_j)
\]
Since $H_{jl} \le 0 $ and $x_l - x_j \le 0$, each term is nonnegative, hence each term must be zero. Therefore, $x_j = x_l$ whenever $H_{jl} \neq 0$. If the comparison graph induced by the nonzero off-diagonal entries of $H$ is connected, all entries of $\mathbf{x}$ are equal, hence $\mathbf{x} = c\bm{1}$. The centering constraint $\bm{1}^\top \mathbf{x} = 0$ forces $c = 0$, so $\mathbf{x} = \bm{0}$. Therefore, $H$ is positive definite on the mean-zero subspace, and the constrained linear system admits a unique solution.

\section{Convergence Analysis}
\label{sec:convergence_analysis}

Let $\mathcal{D}$ be the dataset of pairwise comparisons, $N = |\mathcal{D}|$ denotes the number of datapoints, $K = | \rv|$ the number of items, and $M = |\betav|$ the number of workers. For each worker $j \in [M]$, we constrain $\beta_j \in [-1, 1]$. We assume that the rewards are bounded as $r_i \in [-R, R]$ for all items $l \in [K]$ for some fixed constant $R > 0$.
In each iteration $t$ of the EM algorithm, we independently set
\[
c_i^2 = \kappa_i^{(t)} = \frac{1}{2\eta_i^{(t)}} \tanh\left(\frac{\eta_i^{(t)}}{2}\right), \quad \text{where } \eta_i^{(t)} = \beta_{s_i}^{(t)} \left(r_{w_i}^{(t)} - r_{l_i}^{(t)}\right).
\]
Since $\beta_{s_i} \in [-1,1]$ and $r_{w_i}, r_{l_i} \in [-R,R]$, it follows that $\abs{\eta_i} \le 2R$.\\

The function $f(x) = \frac{\tanh(x/2)}{2x}$ is positive, even, and monotonically decreasing for $x>0$. Therefore, there exists a constant
$$\zeta = \frac{1}{4R} \tanh(R) > 0$$
such that $c_i^2 \in [\zeta, 1/4]$ for all $i$.



For each worker $s$, let $G_s$ denote the comparison graph induced by the worker's pairwise comparisons $\mathcal{D}_s$. We assume that $G_s$ is acyclic.

Finally, we assume that the worker index and the two item indices (denoted by random variables $B, W,$ and $L$, respectively) are independent and uniformly distributed as $B \sim \text{Uniform}([M]), W,L \sim \text{Uniform}([K])$. The comparison outcome is then generated according to \ref{eq:boltz}. Under this model, we can show the following lemma.

\setcounter{lemma}{0}
\begin{lemma}
\label{thm:rip_main}
Define matrix operators $A_i \in \reals^{M \times K}$ for each $i \in [N]$ as follows:
\begin{equation*}
A_i^\top = 2c_i \alpha 
\begin{array}{c@{}c}
 & \begin{array}{c c c c c}
  & & s_i & & 
 \end{array} \\ &
\begin{bmatrix}
0 & 0 & \cdots & 0 & \cdots & 0 \\
\vdots & \vdots & & \vdots & & \vdots \\
0 & 0 & \cdots & +1 & \cdots & 0\\  
\vdots & \vdots & & \vdots & & \vdots \\
0 & 0 & \cdots & -1 & \cdots & 0 \\ 
\vdots & \vdots & & \vdots & & \vdots \\
0 & 0 & \cdots & 0 & \cdots & 0
\end{bmatrix}
\begin{array}{c} 
 \\[-2ex]
 \vdots \\ 
 w_i \\ 
 \vdots \\ 
 l_i \\ 
 \vdots
\end{array}
\end{array}
\end{equation*}
where $+1$ appears at position $(w_i, s_i)$, $-1$ at position $(l_i, s_i)$, with all other entries equal to $0$, and $\alpha = \sqrt{MK/2N}$. Equivalently, $A_i = 2c_i \alpha \ \bm{e_{s_i}} (\bm{e_{w_i}} - \bm{e_{l_i}})^\top$. Let $X \in \reals^{M \times K}$ be any matrix of arbitrary rank $r$ with $X\bm{1} = \bm{0}$, where $\bm{1}$ and $\bm{0}$ denote the all-ones and all-zeros vectors, respectively. Define the linear operator $\opA \colon \reals^{M\times K} \to \reals^N$ as $\opA(X) = (\langle A_1, X\rangle, \dots, \langle A_N, X\rangle)^\top$. Then, for any $\eta > 0$, there exists $\mu > 0$ and $\delta>0$ such that with probability at least $1-\eta$:
\[
(1-\mu)\norm{X}{F}^2 \le \|\opA(X)\|_2^2 \le (1+\mu) \norm{X}{F}^2 \, ,
\]
provided that 
$$N\geq C\;\frac{MK}{\delta^2}\ln\left(\frac{2}{\eta}\right),$$ where $C > 0$ is some constant.
\end{lemma}

\begin{proof}
We first show that $\opA$ satisfies RIP in expectation, and we subsequently bound the deviation of $\|\opA(X)\|_2^2$ from its expectation to achieve the desired result. \\

Consider $\langle A_i, X \rangle^2 = {4c_i^2 \alpha^2} (X_{s_i, w_i} - X_{s_i, l_i})^2$. Under uniform sampling of $B,W,L$ and the assumption $X\bm{1} = \bm{0}$, we have that
\[
\E_{B,W,L} \langle A_i, X \rangle^2 = \E_{B}\left[ \E_{W,L}\left[\langle A_i, X \rangle^2 | B = s_i\right]\right]
\]
Conditioned on $B = s_i$, we get the following expression. 
\begin{align*}
    \E_{W,L}\left[\langle A_i, X \rangle^2\right]
    &= {4c_i^2\alpha^2}\E_{W,L}\left[(X_{s_i,w_i} - X_{s_i,l_i})^2\right]\\
    &= 4c_i^2\alpha^2\left(\E_{W,L} X_{s_i,w_i}^2 + \E_{W,L} X_{s_i,l_i}^2 - 2\E_{W,L} X_{s_i,w_i}X_{s_i,l_i}\right)\\
    &= 8c_i^2\alpha^2 \left(\E_{W,L} X_{s_i,w_i}^2 - \left(\E_{W,L} X_{s_i,w_i}\right)^2\right)
\end{align*}

Now, note that
\[
\E_{W,L} X_{s_i,w_i}^2 = \frac{1}{K}\sum_{j=1}^K X_{s_i,j}^2 \midand \E_{W,L} X_{s_i,w_i} = \frac{1}{K}\sum_{j=1}^K X_{s_i,j},
\]
and define $S_b = \sum_{j=1}^K X_{b,j}$ and $T_b = \sum_{j=1}^K X_{b,j}^2$. Observe that $S_{b} = 0$ for every $b \in [M]$ since $X\bm{1} = \bm{0}$, which implies that $\E_{W,L} X_{s_i,w_i} = 0, \text{and } \E_{W,L}[\langle A_i, X \rangle^2] = {4c_i^2\alpha^2}(\frac{2}{K} T_{s_i})$. Finally, 
\begin{equation}
\label{eq:ex_AiX^2}
\E_{B,W,L} \langle A_i, X \rangle^2 = \frac{1}{M} \left( {4c_i^2\alpha^2}\cdot \frac{2}{K}\sum_{j=1}^M T_{s_j} \right) = \frac{1}{M} \left( {4c_i^2\alpha^2}\cdot \frac{2}{K}\sum_{j=1}^M \sum_{i=1}^K X_{ji}^2 \right) = \frac{1}{M} {4c_i^2\alpha^2}\frac{2}{K} \norm{X}{F}^2 = \frac{4c_i^2}{N} \norm{X}{F}^2,
\end{equation}

where the last equality follows from the definition of $\alpha$. Summing over all $N$ datapoints, we obtain:
\[
\E_{B,W,L} \| \opA(X) \|_2^2 = \sum_{i=1}^N \E_{B,W,L} \langle A_i, X \rangle^2 = \norm{X}{F}^2 \sum_{i=1}^N \frac{4c_i^2}{N}
\]

Note that $4\zeta \le 4c_i^2 \le 1$. Take $\epsilon = 1-4\zeta > 0$, such that
\begin{equation}
\label{eq:expectation_operator_bound}
(1-\epsilon) \norm{X}{F}^2 \leq 4 \zeta \norm{X}{F}^2 \le \E_{B,W,L} \| \opA(X) \|_2^2 \le \norm{X}{F}^2 \le (1+\epsilon) \norm{X}{F}^2.
\end{equation}

Now, we will bound the deviation of $\| \opA(X) \|_2^2$ from its expectation. 
Define random variables $Y_i = \langle A_i, X \rangle^2$; then $\sum_{i=1}^N Y_i = \| \opA(X) \|_2^2$. We will first bound $Y_i$. Note that $Y_i \ge 0$. Further, 
\[
Y_i = \langle A_i, X \rangle^2 = {4c_i^2\alpha^2} (X_{s_i,w_i} - X_{s_i,l_i})^2 \le {4c_i^2\alpha^2}(|X_{s_i,w_i}| + |X_{s_i,l_i}|)^2.
\]
Since each $X_{ij}^2 \le \norm{X}{F}^2$, we have that $|X_{ij}| \le \norm{X}{F}$, and consequently,
\begin{equation}
\label{eq:Y_bounds}
0 \le Y_i \le  {16c_i^2\alpha^2}\norm{X}{F}^2
\end{equation}

Similarly, $\E Y_i^2 \le \max{Y_i} \, \cdot \, \E Y_i$. We know from Equation~\eqref{eq:ex_AiX^2} that $\E Y_i = \E\langle A_i,X \rangle^2 = \frac{4c_i^2}{N}\norm{X}{F}^2$, which gives 
\begin{equation}
\label{eq:EYsqd_bounds}
\E Y_i^2 \le {16c_i^2\alpha^2}\norm{X}{F}^2 \ \cdot \  \frac{4c_i^2}{N}\norm{X}{F}^2 = \frac{64c_i^4\alpha^2}{N}\norm{X}{F}^4 \le \frac{64c_i^2\alpha^2}{N}\norm{X}{F}^4\; ,
\end{equation}
where the final inequality follows from the fact that $c_i < 1$. Define random variables $Z_i = Y_i - \ EY_i$. Then, $Z_i$ are independent and zero-mean, $|Z_i| \le |Y_i| + |\E Y_i| \le 2|Y_i| \le {32c_i^2\alpha^2}\norm{X}{F}^2$, and, $\E Z_i^2 \le \E Y_i^2 \le \frac{64c_i^2\alpha^2}{N}\norm{X}{F}^4$.


Then, using Bernstein's inequality on $Z_i$, we get,
\begin{align*}
\Pr{\left( \left| \sum_{i=1}^N Z_i \right| > \delta\norm{X}{F}^2 \right)} \le 
2 \exp{\left( -\frac{\delta^2 \norm{X}{F}^4/2}{ N \cdot \frac{64c_i^2\alpha^2}{N}\norm{X}{F}^4 + \frac13 \cdot 16\delta c_i^2\alpha^2\norm{X}{F}^4} \right)}
&\le 2 \exp{\left( -\frac{\delta^2 \norm{X}{F}^4}{{80c_i^2\alpha^2}\norm{X}{F}^4} \right)} \\
&= 2 \exp{\left( -\frac{\delta^2}{80c_i^2\alpha^2} \right)}.
\end{align*}
For this to be bounded above by $\eta$, we need 
\[
2 \exp{\left( -\frac{\delta^2}{80c_i^2\alpha^2} \right)} = 2 \exp{\left( -\frac{\delta^2 N}{40c_i^2 MK} \right)} \le \eta \implies
N \ge C\frac{MK}{\delta^2}\log{\left( \frac{2}{\eta} \right)},
\]
for some constant $C > 0$. Therefore, with probability at least $1-\eta$, we get that 
\begin{align*}
\left| \sum_{i=1}^N Z_i \right| \le \delta\norm{X}{F}^2
&\implies \left| \sum_{i=1}^N \langle A_i,X \rangle^2  - \sum_{i=1}^N \E\langle A_i,X \rangle^2\right| \le \delta\norm{X}{F}^2\\ \\
&\implies \left| \|\opA(X)\|_2^2  - \E\|\opA(X)\|_2^2 \right| \le \delta\norm{X}{F}^2 \\\\
&\implies (1-\mu) \norm{X}{F}^2 \le \|\opA(X)\|_2^2 \le (1+\mu)\norm{X}{F}^2
\end{align*}
with $\mu = \epsilon + \delta$.
\end{proof}
We now show that the M-step optimization problem can be written as a matrix sensing problem. Recall that the M-step at iteration $t$ maximizes
\begin{align*}
Q(\bm{\theta} \mid \bm{\theta}^t)
&= \frac{1}{2} \sum_{i=1}^N \left( \beta_{s_i} \left(r_{w_i} - r_{l_i} \right) - c_i^2\beta_{s_i}^2 \left(r_{w_i} - r_{l_i} \right)^2 \right)\\
&= -\frac{1}{2} \sum_{i=1}^N \left( c_i^2\beta_{s_i}^2 \left(r_{w_i} - r_{l_i} \right)^2 - 2\left( \frac{1}{2c_i} \right) c_i\beta_{s_i} \left(r_{w_i} - r_{l_i} \right) + \frac{1}{4c_i^2} - \frac{1}{4c_i^2} \right)\\
&= -\frac{1}{2} \sum_{i=1}^N \left( c_i \beta_{s_i} \left(r_{w_i} - r_{l_i} \right) - \frac{1}{2c_i} \right)^2 + \sum_{i=1}^N \frac{1}{8c_i^2} \\
\end{align*}

Therefore,
\begin{align*}
\arg\max_{\bm{\theta}} Q(\bm{\theta} \mid \bm{\theta}^t)
&= \arg\max_{\bm{\theta}} \left(-\frac{1}{2} \sum_{i=1}^N \left( c_i \beta_{s_i} \left(r_{w_i} - r_{l_i} \right) - \frac{1}{2c_i} \right)^2 + \sum_{i=1}^N \frac{1}{8c_i^2}\right)\\
&= \arg\min_{\bm{\theta}} \sum_{i=1}^N \left( c_i \beta_{s_i} \left(r_{w_i} - r_{l_i} \right) - \frac{1}{2c_i} \right)^2\\
&= \arg\min_{\bm{\theta}} \sum_{i=1}^N \left( 2c_i\alpha\, \beta_{s_i} \left(r_{w_i} - r_{l_i} \right) - \frac{\alpha}{c_i} \right)^2\\
&= \arg\min_{\bm{\theta}} \sum_{i=1}^N \left( \langle A_i, \betav \rv^\top \rangle - \frac{\alpha}{c_i} \right)^2\\
\end{align*}
where the second equality follows from the fact that $c_i$ are constants, and the fourth equality follows from the definition of matrices $A_i$. Define the vector $\mathbf{u}^{(t)}\in\reals^N$ with $u_i^{(t)} = \frac{\alpha}{c_i}$. Then,
\begin{equation*}
\arg\max_{\bm{\theta}} Q(\bm{\theta} \mid \bm{\theta}^t)
    = \arg\min_{\bm{\theta}} \left\| \opA(\betav \rv^\top) - \mathbf{u}^{(t)} \right\|_2^2,
\end{equation*}

We now invoke existing matrix sensing guarantees. 
Let $X^*_t \in \reals^{M \times K}$ denote the underlying true matrix and $X_{1,t}^*$ denote its best rank-1 approximation. Define the distance between a factorization $(\betav, \rv)$ and $X_{1,t}^*$ as
\[
\text{DIST} (\betav, \rv; X_{1,t}^*) = \min_{(\betav^*, \rv^*) \in \mathcal{X}_1^*} \left\| \begin{bmatrix} \betav \\ \rv \end{bmatrix} - \begin{bmatrix} \betav^* \\ \rv^* \end{bmatrix} \right\|_F
\]
where $\mathcal{X}_1^*$ is the set of balanced rank-one factorizations for $X_{1,t}^*$, defined as
\[
\mathcal{X}_1^* = \left\{ (\betav^*, \rv^*) : \betav^* \in \mathbb{R}^{M}, \rv^* \in \mathbb{R}^{K}, \betav^* \rv^{*\top} = X_{1,t}^*, \sigma_1(\betav^*) = \sigma_1(\rv^*) = \sigma_1(X_{1,t}^*)^{1/2}, \right\}.
\]
Then, restating Remark~3 from~\cite{park2016nonsquarematrixsensingspurious},
\begin{remark}
\label{remark:highrank_matrix_sensing}
\textbf{(High-rank matrix sensing)}
Suppose that $X^*_t$ is of arbitrary rank, and let $X_{1,t}^*$ denote its best rank-1 approximation. Let $\mathbf{b} = \opA(X^*)$ and let $(\betav^*, \rv^*)$ be a balanced factorization of $X_{1,t}^*$. If $0 \le \delta_2 \le \delta_4 < 0.005$, then, for any local minimum $(\betav, \rv)$ the distance to $(\betav^*, \rv^*)$ is bounded by
\[
\textsc{DIST} (\betav, \rv; X^*_t) \le \frac{1250}{3\sigma_1(X^*_t)}\cdot \norm{\opA(X^*_t - X_{1,t}^*)}{2}.
\]
\end{remark}
Using this, we can give the following theorem.
\convergence*
\begin{proof}
Define $X^*_t \in \reals^{M\times K}$ as any matrix consistent with the equations $[\mathcal{A}(X^*_t)]_i = u_i^{(t)}$. From the definition of the sensing matrices $A_i$, this implies
\[
u_i^{(t)} = \frac{\alpha}{c_i} = 2c_i \alpha (X^*_{t, \{s_i, w_i\}} - X^*_{t, \{s_i, l_i\}}) \implies X^*_{t, \{s_i, w_i\}} - X^*_{t, \{s_i, l_i\}} = \frac{1}{2c_i^2}
\]
Under the assumption that for each worker $s$, the comparison graph is acyclic, and with the additional constraint that $\sum_{l\in [K]} X^*_{t,\{s_i, l\}} = 0$, this linear system is always consistent. Thus, a solution $X^*_t$ exists.

Now, our operator $\opA$ satisfies RIP on matrices $X$ of arbitrary rank $r$ with $X\bm{1} = \bm{0}$. Because the coefficients $c_i^2 \in [\zeta, 1/4]$, we can define $R = 0.1$, such that $\zeta > 0.249$ and $\varepsilon < 0.004$ and the RIP constants $\mu_r < 0.005$, independent of $r$.

Note that $X^*_t$ may not be rank-1. So, we denote by $X_{1,t}^*$ its best rank-1 approximation. Then, by Remark~\ref{remark:highrank_matrix_sensing}, any local minimum of the objective
\[
\min_{\betav,\rv} \left\| \opA(\betav \rv^\top) - \mathbf{u}^{(t)} \right\|_2^2
\]
satisfies
\[
\text{DIST}(\betav, \rv; X^*_t) \leq \frac{1250}{3\sigma_1(X^*_t)} \|\mathcal{A}(X^*_t - X_{1,t}^*)\|_2.
\]

\end{proof}

\section{Evaluation Metrics: More details}
\label{app:eval_metrics}

\begin{table*}[ht]
\centering
\caption{Summary of Quantitative Metrics. Here, $g_i$ and $s_i$ denote the ground truth and predicted quality scores for item $i \in D_g$, respectively.}
\label{tab:metrics}
\resizebox{0.9\linewidth}{!}{

\begin{tabular}{p{0.2\linewidth}p{0.4\linewidth}p{0.2\linewidth}p{0.2\linewidth}}
\toprule
\textbf{Metric} & \textbf{Formulation} & \textbf{Pros} & \textbf{Cons} \\ \midrule

\textbf{Kendall's Tau} &
$\dfrac{C - D}{\frac{1}{2} k (k-1)}$ &
Standardized rank correlation range; robust to outliers. &
Treats all pairwise swaps with equal importance. \\[6pt]

\textbf{Accuracy} &
$\dfrac{\sum \mathbb{I}(g_i > g_j \land s_i > s_j)}{\sum \mathbb{I}(g_i > g_j)}$ &
Highly interpretable; measures correct ordering probability. &
Does not account for the magnitude of score gaps. \\[6pt]

\textbf{Weighted Acc.} &
$\dfrac{\sum w_{ij}\, \mathbb{I}(g_i > g_j \land s_i > s_j)}{\sum w_{ij}\, \mathbb{I}(g_i > g_j)}$ &
Penalizes significant ranking errors more heavily. &
Metric value depends on the score distribution. \\

\bottomrule
\end{tabular}

}
\begin{flushleft}
\textit{Note: $C$ and $D$ are concordant and discordant pairs; $k = |D_g|$ is the number of items; and $w_{ij} = |g_i - g_j|$ represents the weight of the difference between ground truth scores.}
\end{flushleft}
\end{table*}

Table~\ref{tab:metrics} summarizes the quantitative metrics used to evaluate the predicted item rewards against the ground truth. Kendall's Tau measures the rank correlation between item pairs, providing a standardized indication of agreement while treating all swaps equally. Accuracy computes the probability of correctly ordering item pairs, offering intuitive interpretability but ignoring the magnitude of differences. Weighted Accuracy extends this by penalizing larger ranking errors more heavily, making it sensitive to the score gaps between items. Together, these metrics provide complementary perspectives on ranking performance.

\section{Crowdsourcing Methods Considered}
\label{app:code_methods}
In this section, we evaluate all the representative crowdsourcing methods discussed in Section~\ref{sec:related_work}. These include classical approaches such as the Bradley--Terry--Luce (BT) model~\cite{bradley1952rank}, RankCentrality (RC)~\cite{negahban2012iterative}, CrowdBT~\cite{chen2013pairwise}, FactorBT~\cite{bugakova2019aggregation}, and BARP~\cite{ferrara2024bias}. 

The implementations of RC, CrowdBT, and BARP are obtained from the publicly available repository released by \citet{ferrara2024bias}\footnote{\url{https://github.com/Ambress92/Bias-Aware-Ranker-from-Pairwise-comparisons}}, while BT\footnote{\url{https://github.com/lucasmaystre/choix}} and FactorBT\footnote{\url{https://github.com/Toloka/crowd-kit}} are sourced from their respective open-source repositories. 
These methods differ in how they model pairwise comparisons and annotator reliability. All methods were adapted to PyTorch and optimized for time efficiency.


\section{Implementation Details of BoRaEM}
\label{app:pgem_implementation}

We implement BoRaEM in PyTorch with full GPU support and deterministic seeding for reproducibility. 
Item rewards are initialized to zero, ensuring symmetry across items at the start of optimization. Like CrowdBT \citep{chen2013pairwise}, we initialize worker competencies to a constant positive value (0.8), corresponding to a neutral reliability assumption. This avoids early instability while allowing competency differences to emerge purely from observed comparisons.
Worker competencies are updated via closed-form aggregated statistics using vectorized accumulation. A Gaussian prior with fixed variance is incorporated as ridge regularization by adding prior precision to the denominator of each worker update. This shrinks poorly observed workers toward zero and prevents degeneracy.
Within each EM iteration, reward and competency updates are alternated multiple times to accelerate convergence. 
After each outer iteration, identifiability constraints are enforced: rewards are centered to zero mean and normalized to unit root-mean-square magnitude, with worker competencies rescaled accordingly to preserve the likelihood.
The competency prior variance $\sigma_\beta^2$ imposes a Gaussian prior on worker competencies $\beta_w$, which appears as an additional ridge term in the update:

\[
\beta_s \leftarrow \frac{\sum_{(w,l)\in \mathcal{I}_s} \frac{1}{2}(r_w - r_l)}
{\sum_{(w,l)\in \mathcal{I}_s} \kappa_{wl} (r_w - r_l)^2 + \sigma_\beta^{-2}},
\]

where $\mathcal{I}_s$ is the set of comparisons performed by worker $s$, and $\kappa_{wl}$ are the Polya-Gamma expectations defined in Equation \eqref{eq:kappas}.
A smaller $\sigma_\beta$ increases the prior precision $\sigma_\beta^{-2}$, shrinking $\beta_w$ toward $0$ (strong regularization), whereas a larger $\sigma_\beta$ reduces the prior influence, allowing $\beta_w$ to vary more freely. We set the default value of $\sigma_\beta$ as 1.0.

The reward updates solve a linear system $(H + \lambda_r I) r = b$, where $H$ depends on the worker competencies and Polya-Gamma expectations. 
The regularization coefficient $\lambda_r$ is scaled with the mean of $\kappa_{wl}$ to ensure numerical stability and prevent overfitting to sparse comparisons:
\[
\lambda_r = \lambda_{r,\mathrm{base}} \cdot \max(\bar{\kappa}, 10^{-12}),
\]
where $\bar{\kappa} = \frac{1}{|\mathcal{D}|} \sum_{(w,l)\in \mathcal{D}} \kappa_{wl}$ is the mean Polya-Gamma expectation over all comparisons, 
and $\lambda_{r,\mathrm{base}}$ is a small constant (default $10^{-2}$). 
This acts as a data-dependent ridge, shrinking rewards slightly toward zero, improving convergence stability.

Convergence is monitored via the mean log-likelihood, and optimization terminates when its improvement falls below a fixed tolerance ($10^{-6}$).

Similar to BoRaEM, for HBTL, we set worker competencies to 0.8. Ablation studies are provided in Appendix \ref{sec:pgem_ablation}.

\begin{figure*}[ht]
    \centering
        \begin{subfigure}[t]{0.9\linewidth}
        \centering
        \includegraphics[width=\linewidth]{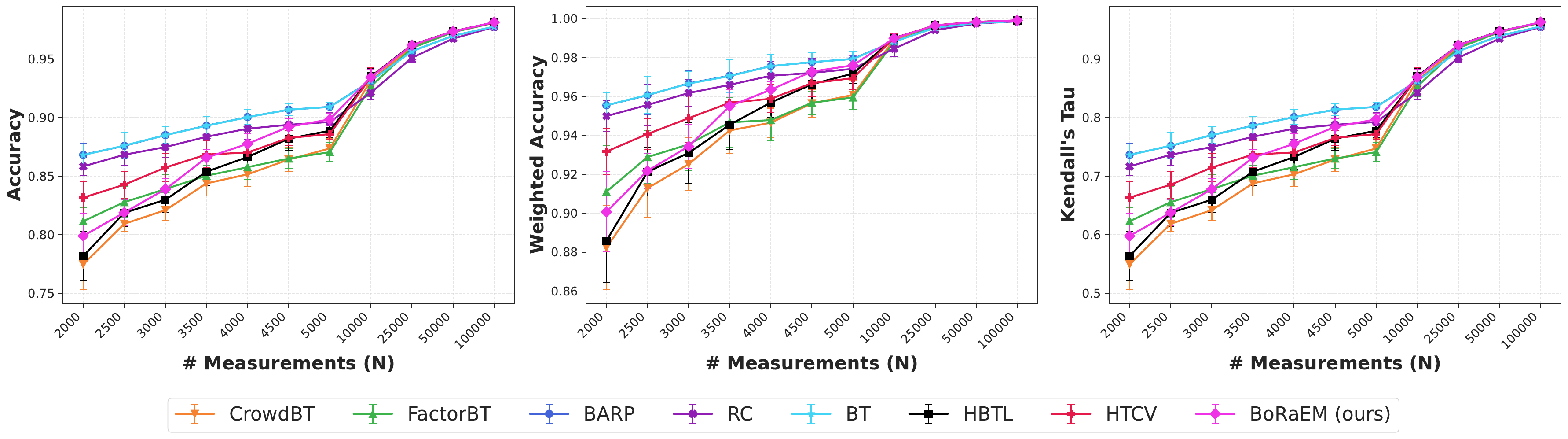}
        \subcaption{Varying \(N\). \(K=100,\,M=500\).}
        \label{fig:sub_vary_N_0_1}
    \end{subfigure}\hfill
    \begin{subfigure}[t]{0.9\linewidth}
        \centering
        \includegraphics[width=\linewidth]{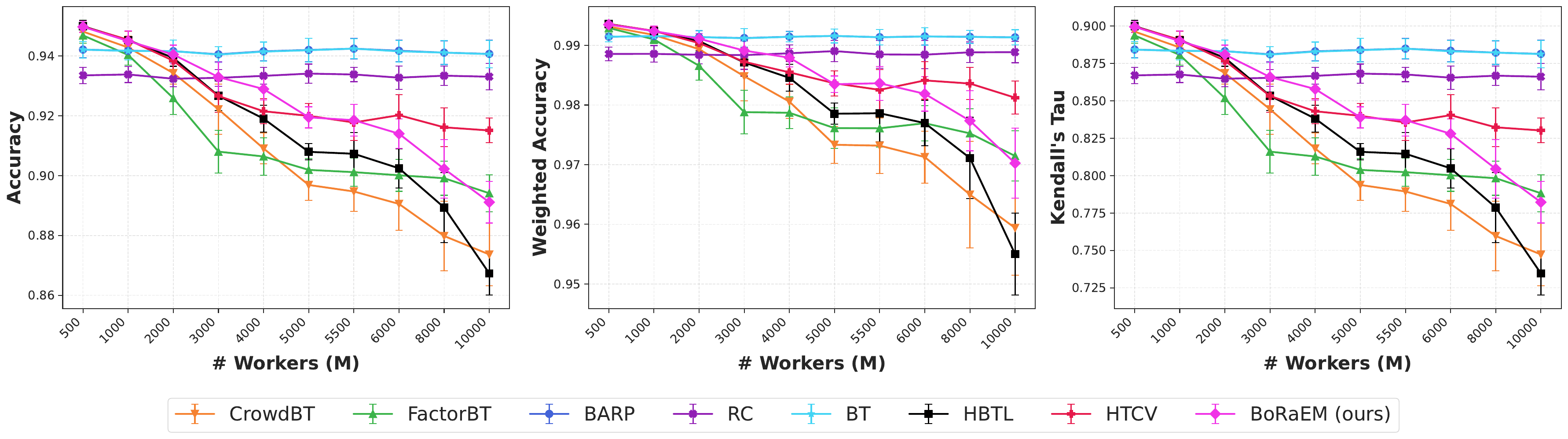}
        \subcaption{Varying \(M\). \(K=200,\,N=25000\).}
        \label{fig:sub_vary_M_0_1}
    \end{subfigure}\hfill
    \begin{subfigure}[t]{0.9\linewidth}
        \centering
        \includegraphics[width=\linewidth]{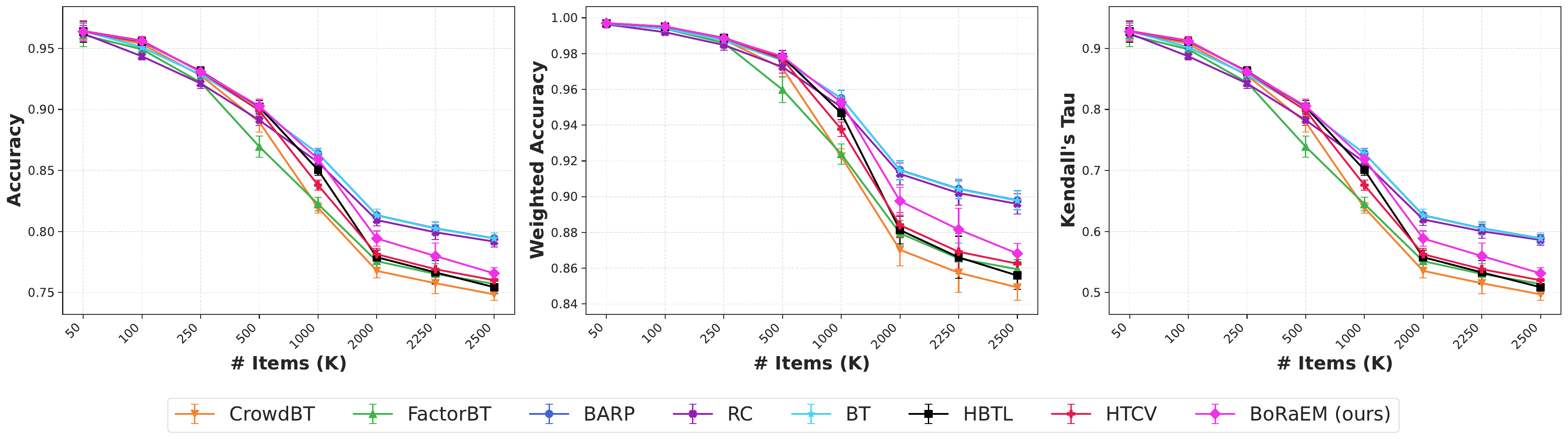}
        \subcaption{Varying \(K\). \(M=20000,\,N=1000\).}
        \label{fig:sub_vary_K_0_1}
    \end{subfigure}
    \caption{Varying \(N\), \(K\), \(M\) individually when \(\beta \in [0,1]\).}
    \label{fig:vary_all_0_1}
\end{figure*}

\section{Synthetic Dataset Experiments: $\beta \in [0,1]$}
\label{app:syn_data_beta_0_1}

In this regime $\beta \in [0,1]$, we do not consider malicious workers ($\beta \in [-1, 0]$). Figure \ref{fig:vary_all_0_1} shows the results when we individually vary each parameter. The details are as follows:

\begin{enumerate}
    \item \textbf{Varying N, Keeping K and M constant:}
We conduct an experiment with fixed values of the number of items $K = 100$ and the number of workers $M = 500$, while varying the number of measurements $N \in \{2k, 2.5k, 3k, 3.5k, 4k, 4.5k, 5k, 10k, 25k, 50k, 100k\}$ to examine the effect of measurement volume on various techniques.
As seen in Figure \ref{fig:sub_vary_N_0_1}, with less measurements the accuracy of methods such as BoRaEM, HBTL and CrowdBT is lower than other methods because more parameters need to be estimated from less number of measurements. Amongst, BoRaEM, HBTL and CrowdBT, we see that BoRaEM and HBTL consistently perform better than CrowdBT in sparser settimngs (e.g. till N = 10000).

\item \textbf{Varying M, Keeping N and K constant:}
We vary the number of workers $M \in \{500, 1k, 2k, 3k, 4k, 5k, 5.5k, 6k, 8k\}$ while fixing the number of items at $K = 200$ and the number of measurements at $N = 25k$, to analyze the impact of worker scale.
As seen in Figure \ref{fig:sub_vary_M_0_1}, HTCV seems to be the most stable alogrithm as compared to BoRaEM, HBTL and CrowdBT.

\item \textbf{Varying K, Keeping N and M constant:}
We vary the number of items $K \in \{50, 100, 250, 500, 1000, 1250, 1500, 1750, 2250, 2500\}$ while fixing the number of workers at $M = 20k$ and the number of measurements at $N = 1k$, in order to study the effect of item scale.
Similar to varying $N$ and $M$, by varying $K$ (Figure \ref{fig:sub_vary_K_0_1}), we observe the same trend where the performance of BoRaEM, HBTL, HTCV and CrowdBT  get affected when number of item increases whereas BoRaEM and HTCV seem to be more robust than CrowdBT.
    
\end{enumerate}

Although Figure \ref{fig:vary_all_0_1} suggests that methods such as \texttt{BTL} and Rank Centrality (\texttt{RC}) are better than BoRaEM, HBTL, HTCV, FactorBT and CrowdBT in sparser settings for $\beta \in [0,1]$, Section \ref{sec:experimental_analysis} (Figure \ref{fig:vary_all_-1_1}) demonstrates how methods that do not model annotator competence can suffer significantly in the presence of adversaries. Section \ref{sec:spammer_analysis} further shows how these methods degrade with the addition of spammers and adversaries in real-world datasets.

\section{Spammer Addition in FaceAge Dataset: More details}
\label{app:spammer_faceage_more}

\begin{figure*}[htbp]
    \centering
    
    \begin{subfigure}[t]{0.9\linewidth}
        \centering
        \includegraphics[width=\linewidth]{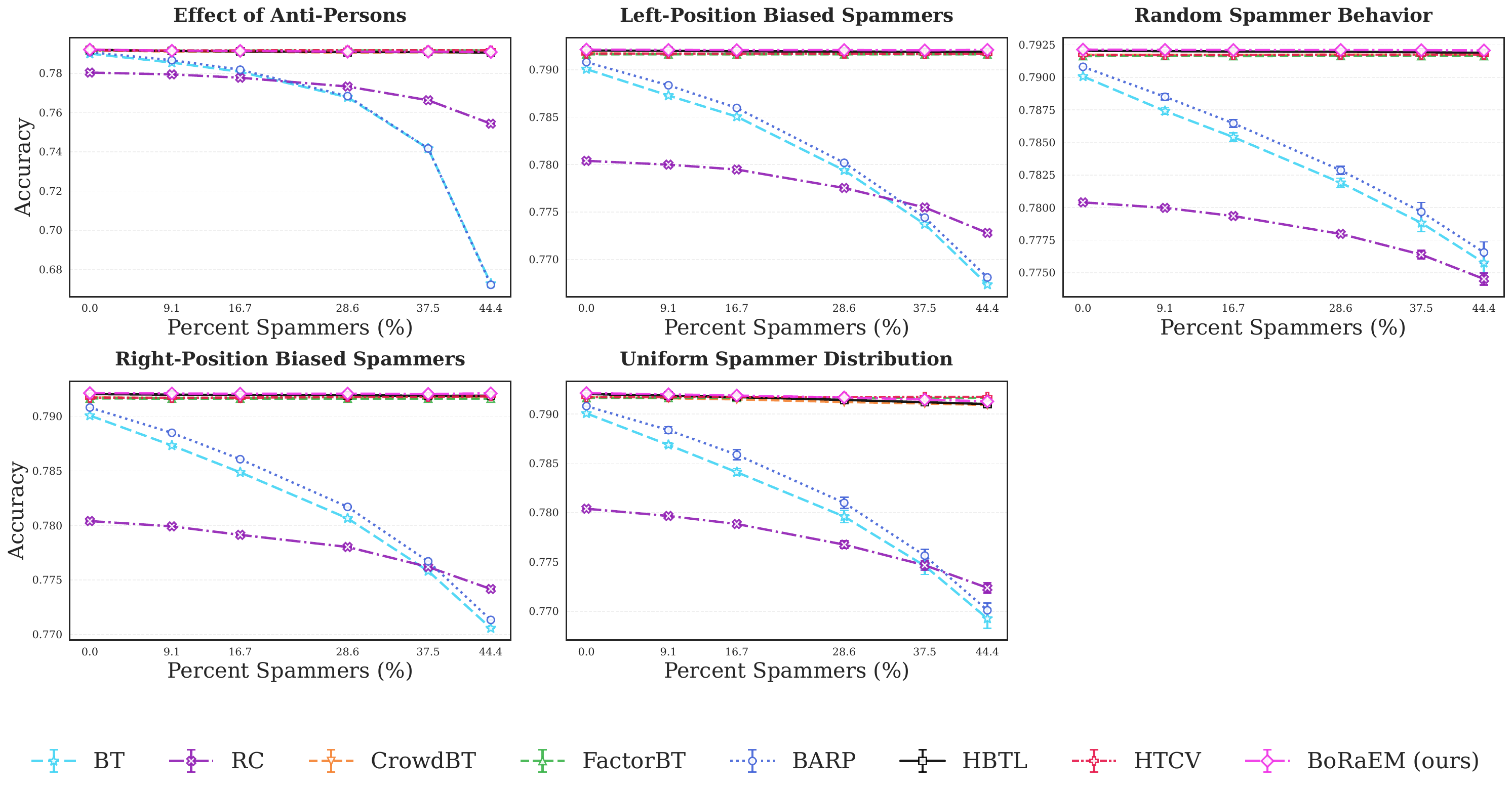}
        \caption{Accuracy ($acc$)}
        \label{fig:spammer_faceage_acc}
    \end{subfigure}
    \hfill
    \begin{subfigure}[t]{0.9\linewidth}
        \centering
        \includegraphics[width=\linewidth]{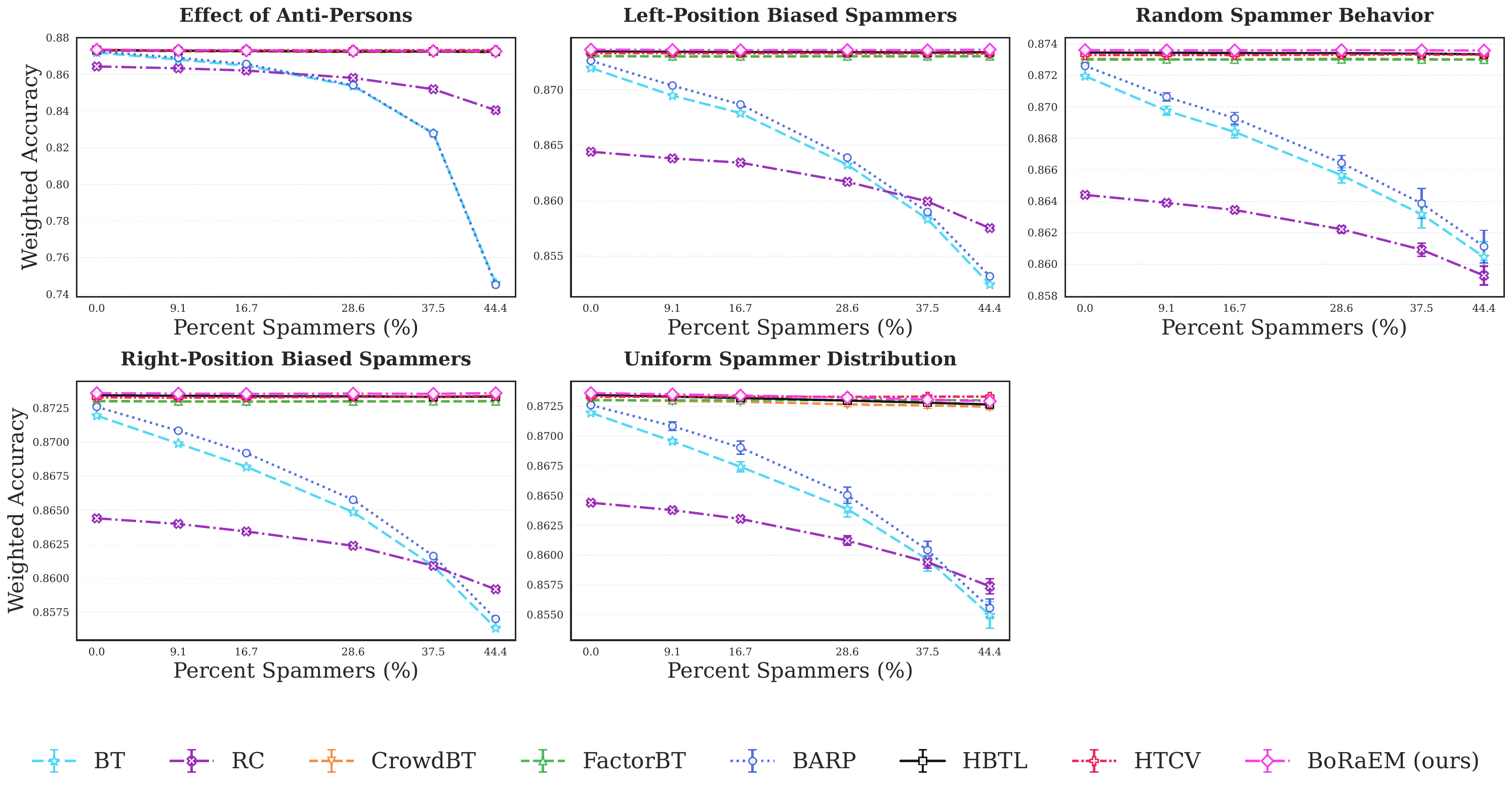}
        \caption{Weighted Accuracy ($wacc$)}
        \label{fig:spammer_faceage_wacc}
    \end{subfigure}
    
    \caption{Performance of different ranking methods under varying proportions of spammers introduced into the \textsc{FaceAge} dataset.}
    \label{fig:spammer_faceage_addition}
    
\end{figure*}

In Section~\ref{sec:spammer_analysis}, we showed that results based on Kendall’s tau demonstrate how methods such as \texttt{FactorBT}, \texttt{BT}, \texttt{BARP}, and \texttt{RC} degrade under spammer addition, whereas \texttt{BoRaEM}, \texttt{HBTL}, and \texttt{CrowdBT} remain robust due to their explicit modeling of annotator competence. Figures \ref{fig:spammer_faceage_acc} and \ref{fig:spammer_faceage_wacc} present the results in terms of accuracy and weighted accuracy when various types of spammers are added to the FaceAge dataset. A similar performance pattern to that observed with Kendall’s tau is evident for these metrics as well.

\begin{figure*}[htbp]
    \centering
    
    \begin{subfigure}[t]{0.8\linewidth}
        \centering
        \includegraphics[width=\linewidth]{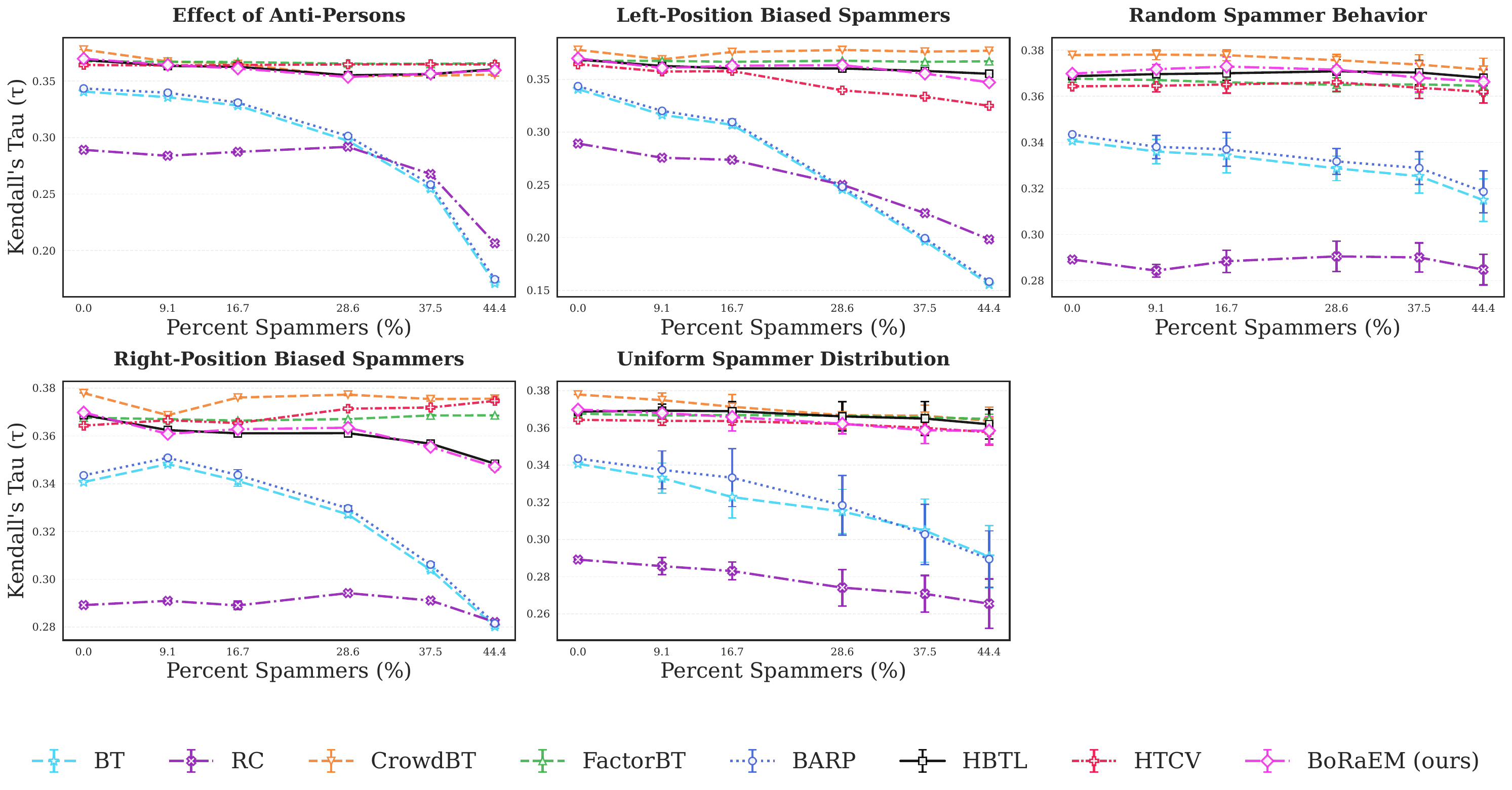}
        \caption{Kendall's tau ($\tau$)}
    \end{subfigure}
    
    \caption{Performance of different ranking methods under varying proportions of spammers introduced into the \textsc{Passage} dataset, measured using Kendall's tau ($\tau$).}
    \label{fig:spammer_passage_tau}
\end{figure*}

\begin{figure*}[htbp]
    \centering
    
    \begin{subfigure}[t]{0.8\linewidth}
        \centering
        \includegraphics[width=\linewidth]{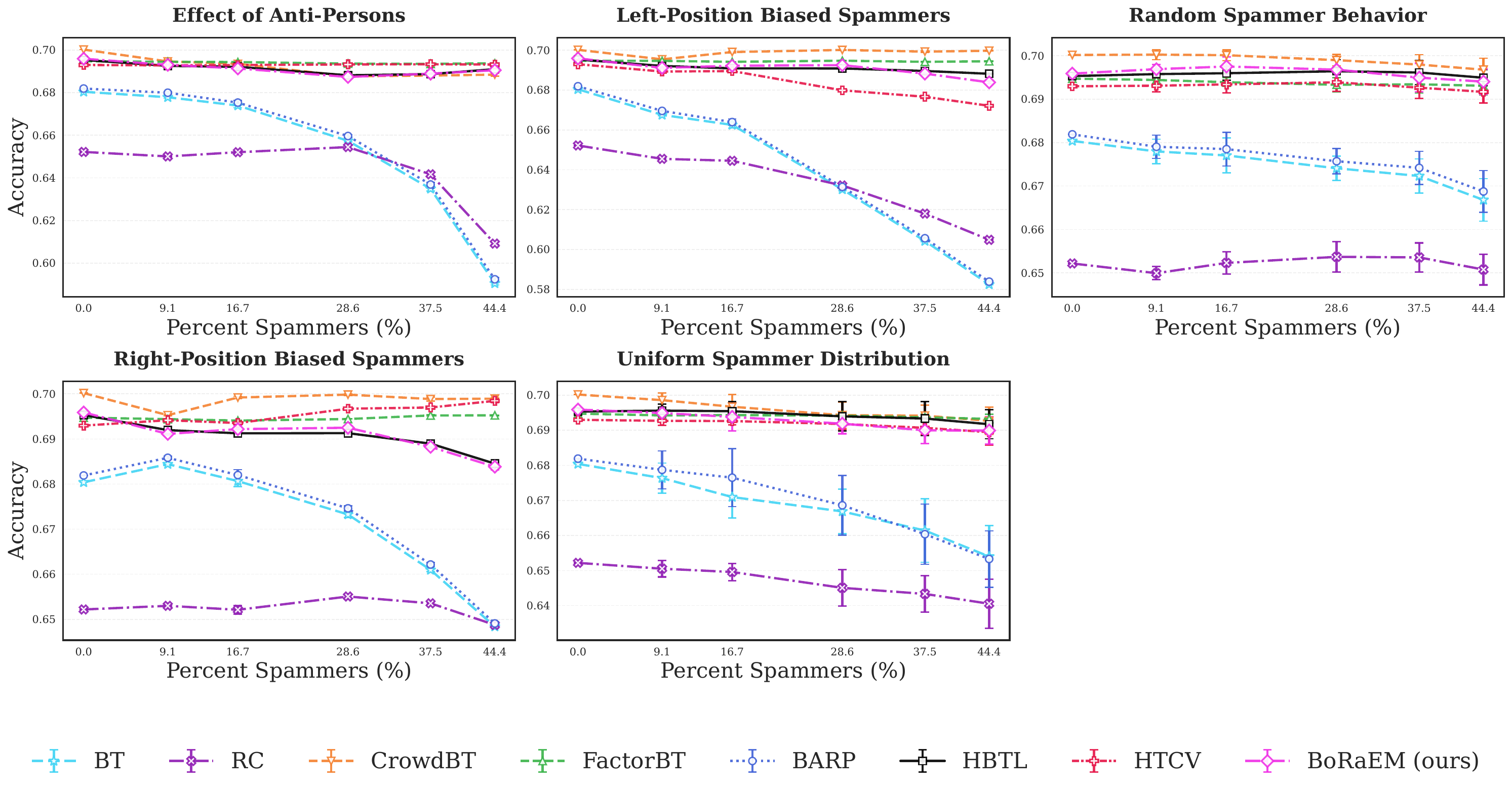}
        \caption{Accuracy ($acc$)}
        \label{fig:spammer_passage_acc}
    \end{subfigure}
    
    
    \begin{subfigure}[t]{0.8\linewidth}
        \centering
        \includegraphics[width=\linewidth]{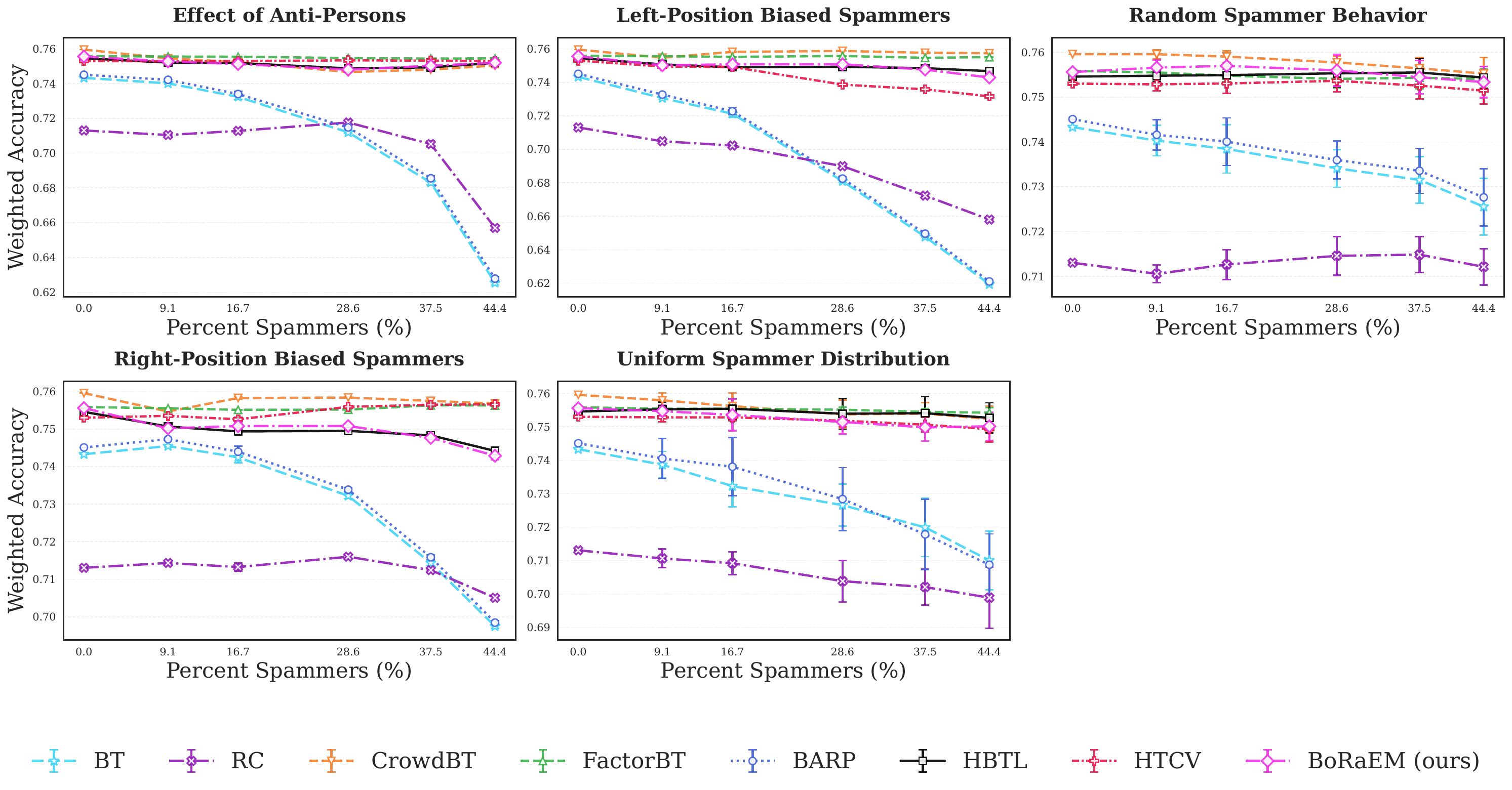}
        \caption{Weighted Accuracy ($wacc$)}
        \label{fig:spammer_passage_wacc}
    \end{subfigure}
    
    \caption{Performance of different ranking methods under varying proportions of spammers introduced into the \textsc{Passage} dataset, measured using Accuracy ($acc$) and Weighted Accuracy ($wacc$).}
    \label{fig:spammer_passage_acc_wacc}
\end{figure*}

\section{Spammer Addition in Passage Dataset}
\label{app:spammer_passage}
In this section, we present the experimental results of spammer addition on the Passage dataset. Figures \ref{fig:spammer_passage_tau} and \ref{fig:spammer_passage_acc_wacc} shows the results using Kendall’s tau as the evaluation metric. \texttt{CrowdBT}, \texttt{HBTL}, \texttt{HTCV}, \texttt{FactorBT} and \texttt{BoRaEM} exhibit robustness to spammers, whereas the remaining methods suffer substantially in their presence.

Similar trends are observed in Figures \ref{fig:spammer_passage_acc} and \ref{fig:spammer_passage_wacc}, which report results in terms of accuracy and weighted accuracy, respectively. 
Models such as BT, FactorBT, Rank Centrality, BARP suffer heavily in presence of spammers and malicious adversaries.
These findings further emphasize the importance of modeling annotator competence or susceptibility to bias in crowdsourced settings.



\section{Ablation Study}
\label{sec:pgem_ablation}

\subsection{Initialization of Worker Competencies}
As discussed in Appendix~\ref{app:pgem_implementation}, we initialize all worker competency values to a fixed positive value (0.8) in case of BoRaEM and HBTL, corresponding to a neutral reliability assumption
to avoid early instability and allowing competency differences to emerge purely from observed comparisons.
Similarly, for CrowdBT, we follow \citet{ferrara2024bias} and initialize worker competencies to 0.7. 

In Section~\ref{sec:related_work}, we discussed about the objective function for FactorBT.
For a comparison between items $d_i$ and $d_j$ by worker $s$, described by a feature vector $\mathbf{x}_{sij} \in \mathbb{R}^M$, where $M$ is the feature dimension. The probability that $d_i$ is preferred over $d_j$ is given  as follows 

\begin{equation}
    \Pr(d_i \succ_s d_j) = \sigma(\gamma_s) \cdot \sigma(r_i - r_j) + (1 - \sigma(\gamma_s)) \cdot \sigma(\langle \boldsymbol{b}_s, \mathbf{x}_{sij} \rangle)
    \label{eq:factorbt_objective}
\end{equation}
where $r_i, r_j$ are the latent scores of items $d_i$ and $d_j$, $\gamma_s$ controls the susceptibility of worker $s$ to bias and $\sigma(\gamma_s)$ can be considered as the the skill of the worker same as that of $\beta_s$ in our formulation, $\boldsymbol{b}_s$ encodes the worker-specific sensitivity to task features, $\sigma(\cdot)$ is the sigmoid function. The model estimates parameters $\{r_i\}$, $\{\gamma_s\}$, and $\{\boldsymbol{b}_s\}$ by maximizing the regularized log-likelihood of the observed pairwise comparisons.
For FactorBT, we uniformly initialize the worker logits ($\gamma_s$)  to $0.7$, which corresponds to initializing worker competencies as $\sigma(0.7) \approx 0.668$.

\paragraph{Effect of initialization.}
In this ablation study, we analyze the effect of initialization on the performance of BoRaEM, HBTL, and CrowdBT on the FaceAge dataset \citep{pavlichenko2021imdb}. The initialization parameter represents annotator competency, denoted by $\beta$.

For BoRaEM and HBTL, we vary $\beta$ from $-1$ to $1$ in increments of $0.1$. For CrowdBT, we vary $\beta$ from $0$ to $1$ in increments of $0.1$, consistent with its probabilistic interpretation.

\begin{figure*}[ht]
    \centering
    \includegraphics[width=\linewidth]{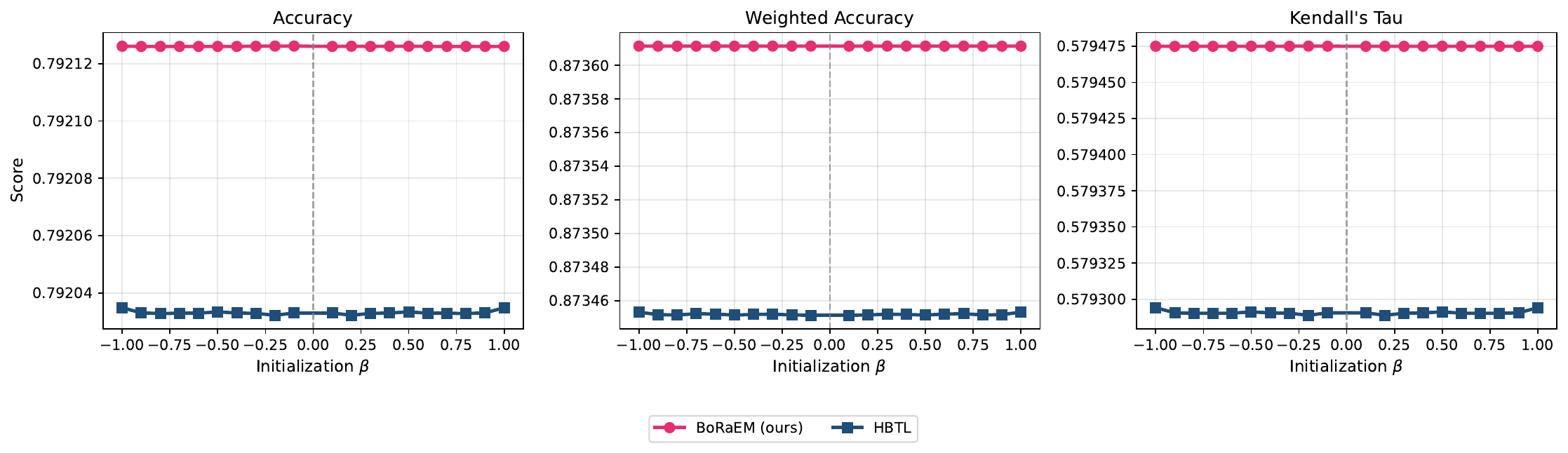}
    \caption{Effect of initialization on BoRaEM and HBTL. Both methods exhibit symmetry around $\beta = 0$, indicating that initializing with competency $\beta$ or $-\beta$ leads to identical performance. This symmetry arises because their likelihood depends on the term $\sigma(\beta_s(r_w - r_l))$, which is invariant under simultaneous sign reversal of competency and score differences.}
    \label{fig:ablation_pgem_gradem}
\end{figure*}

\begin{figure*}[ht]
    \centering
    \includegraphics[width=\linewidth]{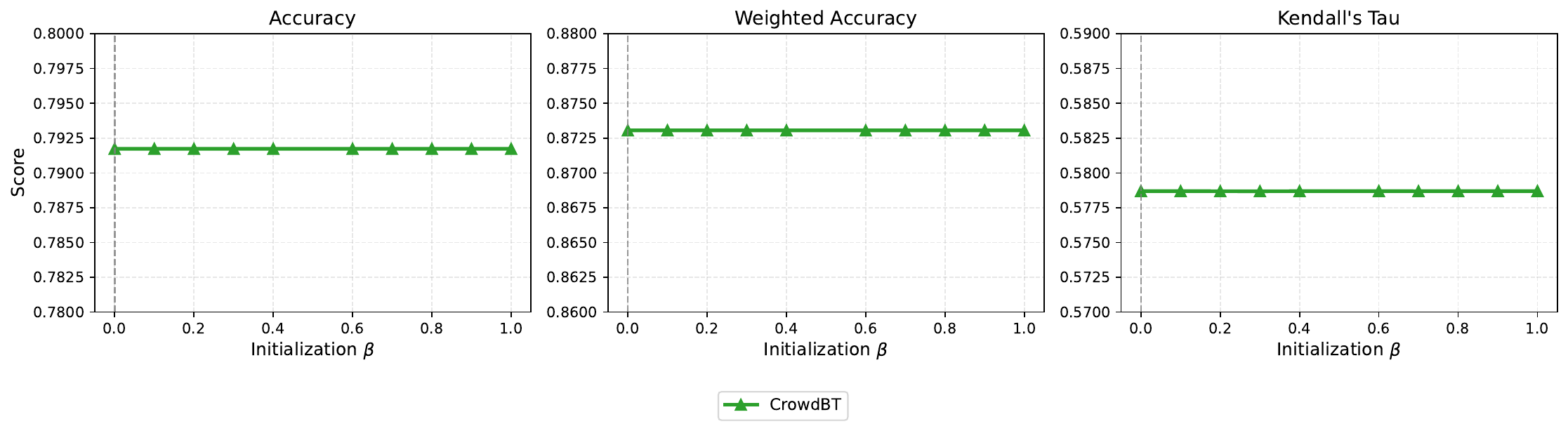}
    \caption{Effect of initialization on CrowdBT. The method exhibits symmetry around $\beta = 0.5$, meaning that initialization with $\beta$ or $1 - \beta$ produces identical results. This follows from the model structure, where $\beta$ and $1-\beta$ represent complementary competency assumptions in the likelihood formulation.}
    \label{fig:ablation_crowdbt}
\end{figure*}

\begin{figure*}[ht]
    \centering
    \includegraphics[width=\linewidth]{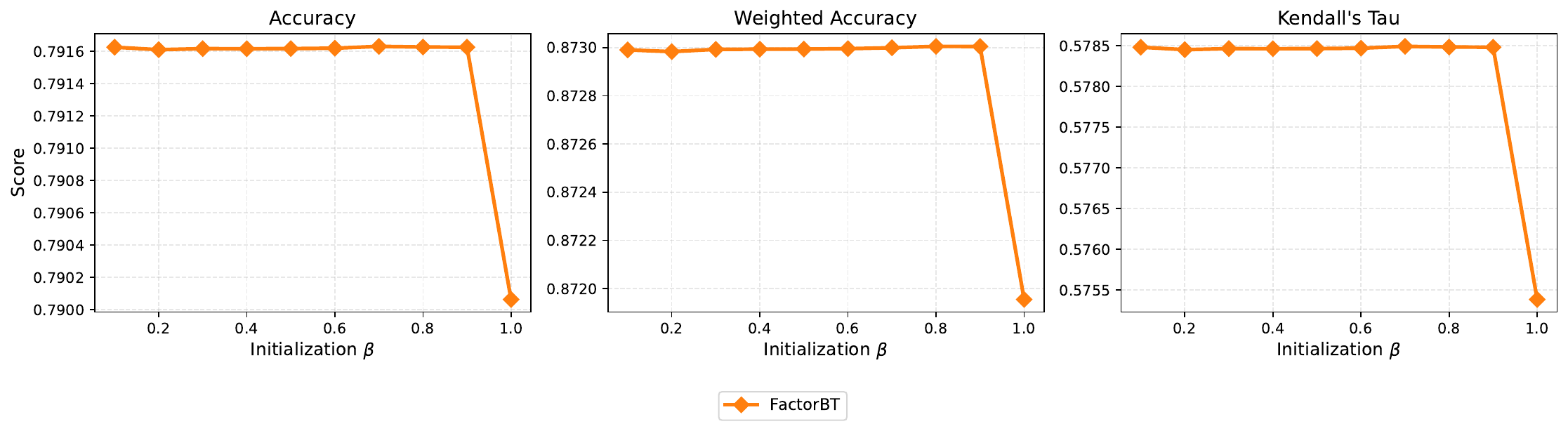}
    \caption{Effect of initialization for FactorBT by varying the initial worker competency, $\beta_s = \sigma(\gamma_s)$.}
    \label{fig:ablation_factorbt}
\end{figure*}

Figures \ref{fig:ablation_pgem_gradem} and \ref{fig:ablation_crowdbt} present ablation studies on the effect of initialization for BoRaEM and HBTL, CrowdBT respectively.

Figure~\ref{fig:ablation_factorbt} presents an ablation study on the effect of initialization for FactorBT by varying the initial worker competency, $\beta_s = \sigma(\gamma_s)$. When $\beta_s$ is initialized to $0$, the component of $\beta_s\sigma(r_i-r_j)$ in Equation~\eqref{eq:factorbt_objective} is completely disabled, reducing the model to a pure worker-bias model that cannot learn meaningful item scores. Conversely, when $\beta_s$ is initialized to $1$, the worker-bias component, $\sigma(\langle \mathbf{b}_s, \mathbf{x}_{sij} \rangle)$, is effectively disabled at initialization, causing the model to behave as the standard Bradley--Terry model. Since $\gamma_s=\operatorname{logit}(\beta_s)$, initializing $\beta_s$ to $1$ yields $\gamma_s\rightarrow\infty$, for which the skill gradients vanish because $\sigma'(\gamma_s)=\sigma(\gamma_s)(1-\sigma(\gamma_s))=0$. For intermediate initializations, the model's performance remains stable, indicating that its performance is largely insensitive to the choice of initialization.



This study demonstrates the robustness of all methods to uniform initialization, where all workers are assigned the same initial competency value, with performance remaining largely stable across the initialization range.

\begin{figure*}[htbp]
    \centering
    \includegraphics[width=0.9\linewidth]{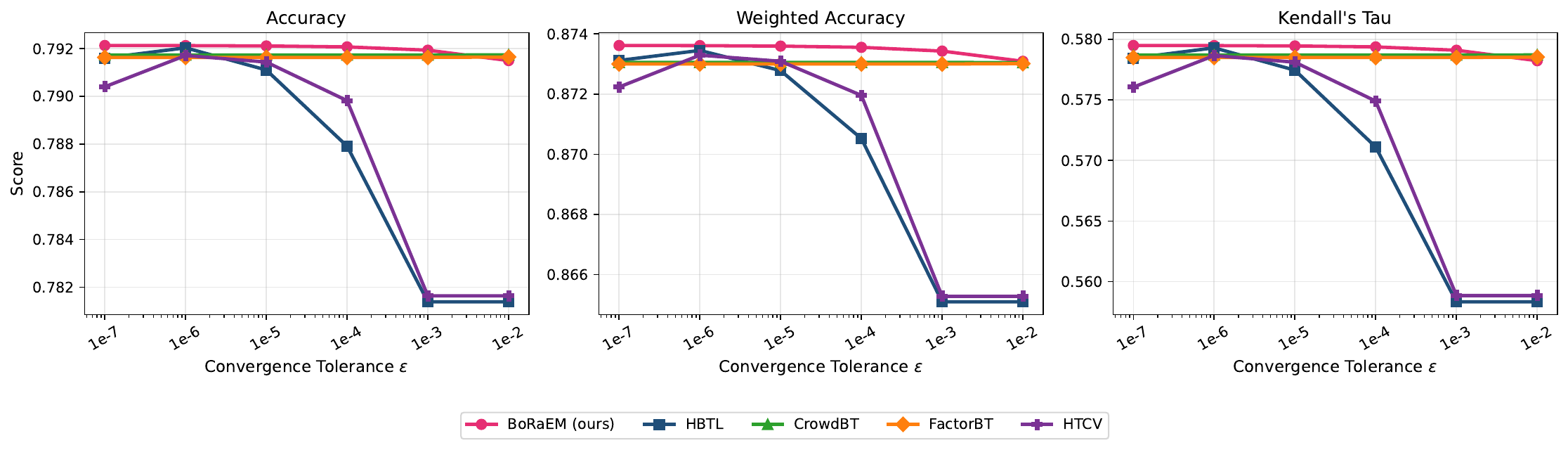}
    \caption{Impact of the mean log-likelihood convergence tolerance ($\epsilon$) on the performance of BoRaEM, HBTL, HTCV, CrowdBT, and FactorBT. The x-axis is shown in decreasing order, from loose to strict convergence. Both methods remain largely stable across the tested range of $\epsilon$, with minor improvements in Accuracy and Kendall's Tau observed for smaller values (stricter convergence). These results demonstrate that the EM updates converge reliably and are not overly sensitive to the choice of stopping criterion, confirming that the default setting of $\epsilon = 10^{-6}$ is within the optimal regime.}
    \label{fig:ablation_epsilon}
\end{figure*}

\subsection{Sensitivity to Mean Log-Likelihood Convergence Tolerance}
As described in Appendix~\ref{app:pgem_implementation}, we set the mean log-likelihood convergence tolerance parameter ($\epsilon$) to $10^{-6}$. 

Figure~\ref{fig:ablation_epsilon} illustrates the sensitivity of BoRaEM, HBTL, HTCV, CrowdBT and FactorBT to different values of the convergence tolerance $\epsilon$. The results show that both methods maintain stable performance across the range, with only minor gains in Accuracy and Kendall's Tau for stricter convergence, supporting the choice of $\epsilon = 10^{-6}$ as a robust default.

\subsection{BoRaEM-Specific Ablation Studies}
\subsubsection{Sensitivity to Competency Prior Variance ($\sigma_\beta$)}
As described in Appendix~\ref{app:pgem_implementation}, we set the Competency Prior Variance parameter ($\sigma_\beta$) to $1.0$. 

\begin{figure*}[htbp]
    \centering
    \includegraphics[width=0.9\linewidth]{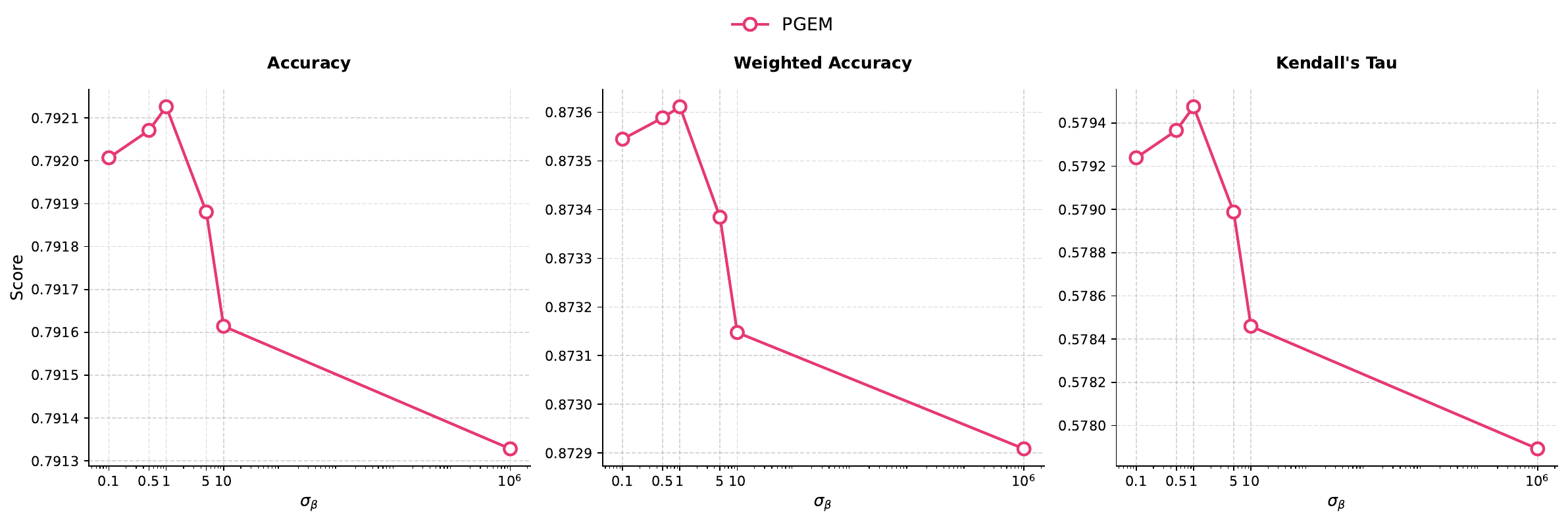}
    \caption{Effect of the competency prior variance ($\sigma_\beta$) on BoRaEM. This hyperparameter controls the strength of the Gaussian prior over worker competencies ($\beta$), where smaller values shrink competencies closer to zero and larger values allow more variability. The results show that both methods are largely stable across a wide range of $\sigma_\beta$, with minor improvements in Accuracy and Kendall's Tau when the prior is moderately loose. This indicates that the EM updates are robust to the choice of $\sigma_\beta$, and the default setting ($\sigma_\beta = 1.0$) provides a reasonable balance between regularization and flexibility.}
    \label{fig:ablation_sigma_beta}
\end{figure*}

The competency prior variance ($\sigma_\beta^2$) controls the strength of the Gaussian prior on worker competencies $\beta_w$ in the EM updates. 
A smaller $\sigma_\beta$ increases the prior precision $\sigma_\beta^{-2}$, which shrinks $\beta_w$ toward $0$ (strong regularization), while a larger $\sigma_\beta$ reduces the prior influence, allowing $\beta_w$ to vary more freely. In the extreme case of $\sigma_\beta \to \infty$, the prior is effectively removed, which can lead to overfitting and slightly lower stability in the estimates. The results above reflect this trade-off: moderate $\sigma_\beta$ (e.g., $1.0$) balances regularization and flexibility, producing optimal Accuracy and Kendall's Tau.

\subsubsection{Sensitivity to Reward Regularization Coefficient ($\lambda_{r,\mathrm{base}}$)}

As described in Appendix~\ref{app:pgem_implementation}, we set the Reward Regularization Coefficient Parameter ($\lambda_{r,\mathrm{base}}$) to $10^{-2}$. 

BoRaEM demonstrates remarkable robustness to the choice of $\lambda_{r,\mathrm{base}}$ except at extreme values. This relative insensitivity suggests that while the regularization is essential for numerical conditioning in the linear system, the model's primary predictive power is driven by the underlying Pólya-Gamma latent variable framework rather than sensitive hyperparameter tuning.

\begin{figure*}[htbp]
    \centering
    \includegraphics[width=0.9\linewidth]{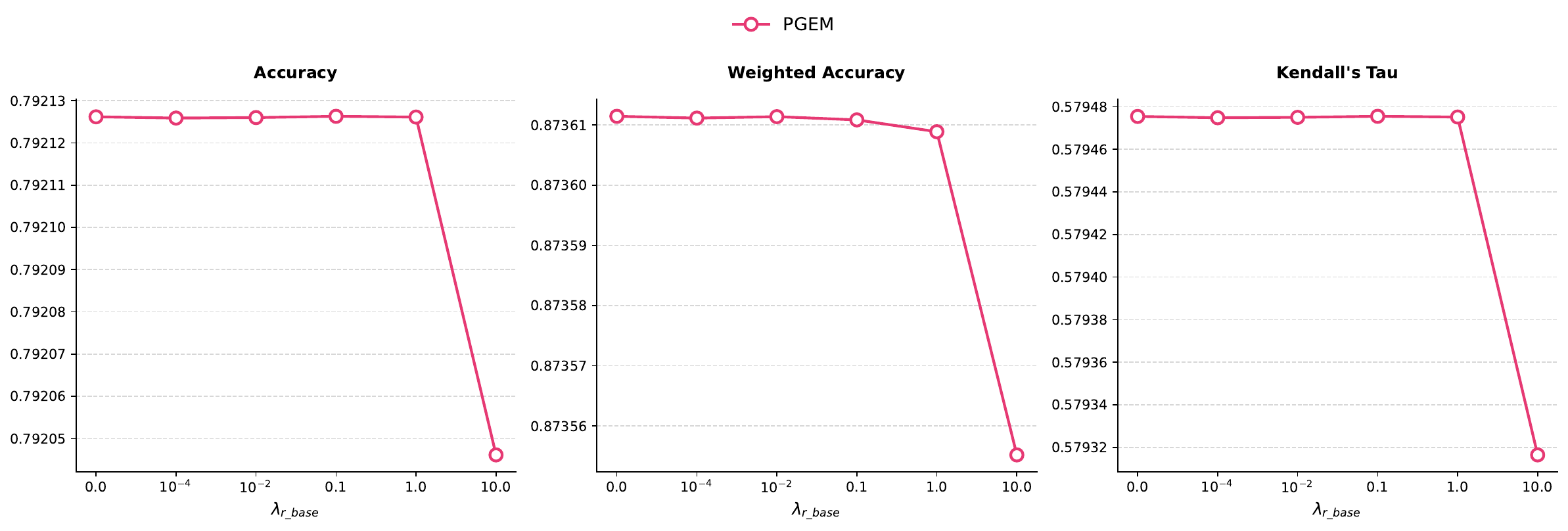}
    \caption{Effect of the reward regularization coefficient $\lambda_{r,\mathrm{base}}$ on BoRaEM performance.}
    \label{fig:ablation_lambda_r}
\end{figure*}

\section{Hardware Details}
\label{sec:hardware}
The experiments were conducted on a server with dual-socket Intel Xeon Gold 6348 CPUs (28 cores / 56 threads per socket, 2.6\,GHz base, 3.5\,GHz turbo, AVX-512), 503\,GiB RAM, and 4 NVIDIA A100 GPUs. The L1, L2, and L3 caches are 2.6\,MiB, 70\,MiB, and 84\,MiB, respectively, across 2 NUMA nodes. The system runs Ubuntu 22.04 LTS.

\end{document}